\documentclass[11pt]{article}
\usepackage[utf8]{inputenc}
\pdfoutput=1

\title{Masked prediction tasks: a parameter identifiability view}
\author{
    Bingbin Liu\thanks{Carnegie Mellon University, \href{bingbinl@cs.cmu.edu}{bingbinl@cs.cmu.edu}} \and 
    Daniel Hsu \thanks{Columbia University, \href{djhsu@cs.columbia.edu}{djhsu@cs.columbia.edu}} \and 
    Pradeep Ravikumar \thanks{Carnegie Mellon University, \href{pradeepr@cs.cmu.edu}{pradeepr@cs.cmu.edu}} \and
    Andrej Risteski \thanks{Carnegie Mellon University, \href{aristesk@andrew.cmu.edu}{aristesk@andrew.cmu.edu}}
}
\date{\vspace{-1em}}

\usepackage{amsmath}
\usepackage{amssymb}
\usepackage{amsthm}
\usepackage[ruled]{algorithm2e}
\usepackage{graphicx}
\usepackage{bm,bbm}

\usepackage{listings}
\usepackage[dvipsnames]{xcolor}
\usepackage[colorlinks = true,
            linkcolor = blue,
            urlcolor  = blue,
            citecolor = blue,
            anchorcolor = blue]{hyperref}
\usepackage{geometry}
\usepackage{enumitem}
\usepackage{url}
\usepackage{fullpage}

\usepackage{caption}
\usepackage{subcaption}

\usepackage{breakcites}
\usepackage{breakurl}

\usepackage[bottom]{footmisc}
\usepackage{natbib}
\usepackage{romannum}
\def \rnum#1{(\romannum{#1})}
\def \eqLabel#1{\overset{(\romannum{#1})}}

\usepackage{environ}

\newif\ifsolution
\solutiontrue
\newcounter{solnequation}
\newcounter{solneqtmp}
\NewEnviron{soln}{
    \setcounter{solneqtmp}{\value{equation}}
    \setcounter{equation}{\value{solnequation}}
    
    \ifsolution\color{red}\expandafter\textbf{Solution} \BODY\fi%
    \setcounter{solnequation}{\value{equation}}
    \setcounter{equation}{\value{solneqtmp}}
    }

\newtheorem{theorem}{Theorem}
\newtheorem*{namedtheorem}{\theoremname}
\newcommand{\theoremname}{testing}

\newtheorem{lemma}{Lemma}
\newtheorem*{lemma*}{Lemma}

\newtheorem{claim}{Claim}
\newtheorem*{claim*}{Claim}

\newtheorem{proposition}[theorem]{Proposition}

\newtheorem*{question*}{Question}

\theoremstyle{definition}

\newtheorem*{definition*}{Definition}

\newtheorem*{remark*}{Remark}

\theoremstyle{plain}

\newtheorem{assumption}{Assumption}

\def\eqref#1{equation~\ref{#1}}

\def\1{\bm{1}}

\def\ve{{{e}}}

\def\vu{{{u}}}
\def\vv{{{v}}}

\def\vx{{{x}}}

\def\mA{{{A}}}
\def\mB{{{B}}}
\def\mC{{{C}}}
\def\mD{{{D}}}

\def\mH{{{H}}}
\def\mI{{{I}}}

\def\mM{{{M}}}

\def\mO{{{O}}}
\def\mP{{{P}}}

\def\mR{{{R}}}

\def\mT{{{T}}}
\def\mU{{{U}}}

\def\mW{{{W}}}
\def\mX{{{X}}}

\DeclareMathAlphabet{\mathsfit}{\encodingdefault}{\sfdefault}{m}{sl}
\SetMathAlphabet{\mathsfit}{bold}{\encodingdefault}{\sfdefault}{bx}{n}

\def\gA{{\mathcal{A}}}

\def\gP{{\mathcal{P}}}

\def\gX{{\mathcal{X}}}

\newcommand{\E}{\mathbb{E}}

\newcommand{\R}{\mathbb{R}}

\newcommand{\softmax}{\mathrm{softmax}}

\newcommand\States{\mathcal{H}}
\newcommand\Obs{\mathcal{X}}

\def \GHMM {\ensuremath{\text{G-HMM}}\xspace}

\def \det {\text{det}}
\def \diag {\text{diag}}
\def \rank {\text{rank}}

\def \mTrans {\mT}
\def \mEmiss {\mO}
\def \mMeans {\mM}

\def \path {\mathsf{path}}

\def \Tensor {\mW}
\def \tRank {R} 
\def \kRank {r} 

\def \tOne {{t_1}}
\def \tTwo {{t_2}}
\def \tThr {{t_3}}

\def \ei {e_i}
\def \ej {e_j}

\def \softmax {\text{softmax}}
\def \simplex {\Delta}

\usepackage{times}

\def\coltSpace{1}

\begin{document}

\maketitle


\pagenumbering{arabic}

\begin{abstract}%
  
  The vast majority of work in self-supervised learning, both theoretical and empirical (though mostly the latter), have largely focused on recovering good features for downstream tasks \citep{Saunshi20LM,wei2021pretrained}, with the definition of ``good'' often being intricately tied to the downstream task itself.  This lens is undoubtedly very interesting, but suffers from the problem that there isn't a ``canonical'' set of downstream tasks to focus on---in practice, this problem is usually resolved by competing on the benchmark dataset du jour. 
  
  In this paper, we present an alternative lens: one of parameter identifiability. More precisely, we consider data coming from a parametric probabilistic model, and train a self-supervised learning predictor with a suitably chosen parametric form. Then, we ask whether we can read off the ground truth parameters of the probabilistic model from the optimal predictor. We focus on the widely used self-supervised learning method of predicting masked tokens, which is popular for both natural languages~\citep{bert} and visual data~\citep{MAE21}.
  
  While incarnations of this approach have already been successfully used for simpler probabilistic models (e.g. learning fully-observed undirected graphical models \citep{ravikumar2010high}), we focus instead on latent-variable models capturing sequential structures---namely Hidden Markov Models with both discrete and conditionally Gaussian observations. We show that there is a rich landscape of possibilities, out of which some prediction tasks yield identifiability, while others do not. Our results, borne of a theoretical grounding of self-supervised learning, could thus potentially beneficially inform practice. Moreover, we uncover close connections with uniqueness of tensor rank decompositions---a widely used tool in studying identifiability through the lens of the method of moments \citep{anandkumar2012multiview}.  
%
%
\end{abstract}

\section{Introduction}


Self-supervised learning is a relatively new approach to unsupervised learning, where the learning algorithm automatically generates auxiliary labels for a given unlabeled dataset, without the need for human annotators, and learns predictors to predict these auxiliary labels. The hope is if the generated prediction task is challenging enough, a successfully learned predictor captures something useful about the underlying data. 
One of the most popular and best-performing self-supervised tasks is the \textit{masked prediction task}, wherein a model is trained to predict a missing part of a sample given the rest of the sample (e.g. predicting masked out words in a sentence).  
Notable early work in such masked prediction based self-supervised learning include word embedding training methods like  word2vec~\citep{word2vec}. 
More recently, architectural advances like the introduction of Transformers~\citep{vaswani2017attention} together with such masked prediction have shown even great progress in learning representations that can perform well on a wide range of downstream tasks~\citep{bert,GPT2,MAE21}, as well as seeming learning structural properties of the data, e.g. syntax~\citep{hewitt2019structural}.  




While self-supervised learning is enjoying a rapid growth on the empirical front, theoretical understanding of why and when it works are nascent at best. In no small part, this is because formalizing the desired guarantees seems challenging. If the goal is to learn ``good features'', how do we mathematically characterize the datasets of interest? In practice, there are benchmarks for the domains of interest~\citep{wang2018glue,wang2019superglue,deng2009imagenet,VTAB,tamkin2021dabs}, on which the features are compared. In theory, one usually has to make very strong assumptions on the relationship between the self-supervised prediction task and the downstream tasks \citep{arora2014new, wei2021pretrained}. If the goal of the predictor is to learn ``something'' about the data distribution, how do we mathematically formalize this?   

In this work, we take a natural \emph{parameter identification} view of this question. More precisely, assuming the data comes from a ground truth probabilistic model in some parametric class, we are interested in knowing whether the parameters of the optimal predictor (when suitably parameterized) can recover the ground-truth parameters of the probabilistic model. 
This strategy (albeit under different names) has been applied in some classical settings. 
For instance, \cite{ravikumar2010high} use \emph{pseudo-likelihood} approaches to learn fully-observed undirected graphical models by predicting a node given its neighbors. Much subsequent work on learning Ising models (in both the ferromagnetic and antiferromagnetic regime) is based on (more complicated) incarnations of these ideas \citep{bresler2015efficiently, vuffray2016interaction}.

In contrast to these prior works,
we tackle this question for the more  complicated case of \emph{latent-variable models} with a time-series structure, specifically Hidden Markov Models with discrete latent and observables (HMMs), as well as Hidden Markov Models with discrete latents and continuous observables that are conditionally Gaussian given the latent (which we will abbreviate as {\GHMM}s).
By contrasting these two setups and looking at various prediction tasks, we show that parameter identifiability from masked prediction tasks is subtle and highly specific to the interaction between the prediction task and the generative model.
For instance, while a single pairwise prediction task cannot recover parameters of an HMM, it can for a \GHMM~due to a different parametric form of the optimal predictor.

Conceptually, there are two main takeaways from our work. 
\vspace{-0.6em}
\begin{itemize}[leftmargin=*]
\item Whether a particular masked prediction task yields parameter identifiability for the data-model is highly dependent on the interaction between the prediction task and the data model. For instance, predicting the conditional mean for a discrete HMM does not yield identifiability, but does so when data comes from a \GHMM. Moreover, the identifiability in the latter case quite strongly leverages structural properties of the posterior of the latent variables. (Section~\ref{s:pairwise}).
\item We demonstrate that tools characterizing the uniqueness of tensor decompositions (e.g., Kruskal's Theorem~\citep{kruskal77,allman2009identifiability}) can be leveraged to prove identifiability. These results are the basis of many efficient algorithms for latent-variable model learning based on the method-of-moments~\citep{mossel2005learning,anandkumar2012multiview,anandkumar2014tensor}. 
For both HMM and \GHMM, we show that if we have predictors
of the tensor product of tokens, e.g., $\mathbb{E}[x_2 \otimes x_3 | x_1]$,
we can carefully construct a 3-tensor, whose rank-1 components are uniquely determined and reveal the parameters of the model. 
\end{itemize}
\vspace{-0.5em}

The rest of the paper is structured as follows.
Section \ref{sec:setup} provides relevant definitions, preliminaries and assumptions.
Section \ref{sec:results} states the main results of this work.
We provide main proofs, including the identifiability proof via tensor decomposition, in Section \ref{sec:proofs}, and defer the rest of the proofs to the appendix.
Finally, we emphasize that this work is a first-cut study on parameter recovery via masked prediction tasks and does not aim to be exhaustive.
We believe there are many interesting open directions and list a few promising next steps at the end.

\vspace{-0.6em}
\section{Setup}
\label{sec:setup}

This work focuses on two classes of latent-variable sequence models. The first are fully discrete hidden Markov models (HMMs), and the second are HMMs whose observables marginally follow a mixtures of Gaussians with identity covariance.
We will denote the observations and hidden states by $\{x_t\}_{t\geq 1}$ and $\{h_t\}_{t \geq 1}$, respectively.
The hidden states $h_1 \to h_2 \to \dotsb$ form a Markov chain, and conditional on $h_t$, the observable $x_t$ is independent of all other variables.
Throughout, we refer to the observations $\{x_t\}_{t\geq 1}$ as tokens, following the nomenclature from language models.

\vspace{-0.6em}
\subsection{Models}

\paragraph{Discrete Hidden Markov Model}

We first describe the parameterization of the standard HMMs with discrete latents and observations.
Let $\Obs := \{1, \dotsc, d\} = [d]$ denote the observation space,
and let $\States := [k]$ be the state space.%
\footnote{Our results will assume $d \geq k$; see Section~\ref{sec:assumptions}.}
The parameters of interest are the \textit{transition matrix} $\mTrans \in \R^{k \times k}$ and the \textit{emission matrix} $\mEmiss \in \R^{d \times k}$,
defined in the standard way as
\begin{equation*}
    P(h_{t+1}=i \mid h_t=j) = \mTrans_{ij}, \qquad
    P(x_{t}=i \mid h_t=j) = \mEmiss_{ij}.
\end{equation*}

\paragraph{Conditionally-Gaussian HMM (\GHMM)}


We next describe the parameterization of \emph{conditionally-Gaussian HMMs ({\GHMM}s)}.
The state space $\States := [k]$
is the same as in the previous case, while the observation space is now continuous with $\Obs := \R^d$.
The parameters of interest are $\mTrans \in \R^{k \times k}$,  the transition matrix, and $\{\mu_i\}_{i \in [k]} \subset \R^{d}$, the means of the $k$ identity-covariance Gaussians.
Precisely,
\begin{equation*}
    P(h_{t+1}=i \mid h_t=j) = \mTrans_{ij}, \qquad
    P(x_t = x \mid h_t = i) = (2\pi)^{-\frac{d}{2}} \exp\big(-\|x-\mu_i\|^2/2\big).
\end{equation*}
We use $\mMeans := [\mu_1, \dotsc, \mu_k] \in \R^{d \times k}$ to denote the matrix whose columns are the Gaussian means.

\subsection{Masked prediction tasks}

We are interested in the (regression) task of predicting one or more ``masked out'' tokens as a function of another observed token, with the goal of minimizing expected squared loss under a distribution given by an HMM or \GHMM. In the case of the discrete HMMs, we will specifically be predicting the \textit{one-hot encoding vectors} of the observations. Thus, both for HMM and \GHMM, predicting a single token will correspond to predicting a vector.
For notational convenience, we will simply associate the discrete states or observations via their one-hot vectors $\{ e_1, e_2, \dotsc\}$ in the appropriate space and interchangeably write $h = i$ and $h = \ei$.
For the task of predicting the tensor product of (one-hot encoding vectors of) tokens $\otimes_{\tau \in \mathcal{T}}\, \vx_{\tau}$ from another token $\vx_t$ (where $\mathcal{T}$ is some index set and $t \notin \mathcal{T}$),
the optimal predictor with respect to the squared loss calculates the conditional expectation:
\begin{equation*}
    f(\vx_t) 
    = \E[\otimes_{\tau \in \mathcal{T}}\, \vx_{\tau}  \mid \vx_t] \in (\mathbb{R}^d)^{\otimes |\mathcal{T}|}.  
\end{equation*}
We use the shorthand
``$\otimes_{\tau \in \mathcal{T}} \vx_{\tau} | x_t$''
to refer to this prediction task.
For instance, consider the case of predicting $x_2$ given $x_{1}$ under the HMM with parameters $(\mEmiss,\mTrans)$.
The optimal predictor, denoted by $f^{2|1}$, can be written in terms of $(\mEmiss,\mTrans)$ as
\begin{equation*}
\begin{split}
  &f^{2|1}(x)
  = \E[x_2 \mid x_1=x] = \sum_{i \in [k]} \E[x_2 \mid h_2=i] P(h_2=i \mid x_1=x)
  \\
  =& \sum_{i \in [k]} \sum_{j \in [k]} \E[x_2 \mid h_2=i] P(h_2=i \mid h_1=j) \underbrace{P(h_1=j \mid x)}_{:= [\phi(x)]_j}
  = \sum_{i \in [k]} \sum_{j \in [k]} \mEmiss_{i} \mTrans_{ij} \underbrace{\frac{\mEmiss_{x,j}}{\sum_{l \in [k]} \mEmiss_{x,l}}}_{:= [\phi(x)]_j} .
\end{split}
\end{equation*}
%
Here $\phi: \R^d \rightarrow \R^k$ denotes the posterior distribution of a hidden state $h_t$ given the corresponding observation $x_t$, i.e., $\phi(x_t) = \E[h_t  \mid  x_t]$.\footnote{The computation here relies on Assumption~\ref{assump:latent}, given in Section~\ref{sec:assumptions}.}

Our goal is to study the parameter identifiability from the prediction tasks, when the predictors have the correct parametric form.
Formally, we define identifiability from a prediction task as follows:
\ifx\coltSpace\undefined
    \vspace{-0.8em}
\fi 
\begin{definition*}[Identifiability from a prediction task, HMM]
  A prediction task suffices for identifiability if, for any two HMMs with parameters $(\mEmiss,\mTrans)$ and $(\tilde\mEmiss,\tilde\mTrans)$, equality of their optimal predictors for this task implies that there is a permutation matrix $\Pi$ such that $\mEmiss = \tilde\mEmiss \Pi$ and $\mTrans = \Pi^\top \tilde\mTrans \Pi$.
\end{definition*}
\vspace{-0.6em}
In other words, the mapping from (the natural equivalence classes of) HMM distributions to optimal predictors for a task is injective.%
\footnote{The parameters are only identifiable from the distribution of observables up to a permutation of the hidden state labels.}
By identifiability from a collection of prediction tasks, we refer to the injectiveness of the mapping from HMM distributions to the collections of optimal predictors for the tasks.
Identifiability for {\GHMM}s is defined analogously with $\mEmiss,\tilde\mEmiss$ changed to $\mMeans,\tilde\mMeans$.

\vspace{-0.6em}

\subsection{Assumptions}
\label{sec:assumptions}

We now state the assumptions used in our results.
The first assumption is that the transition matrices of the HMMs are doubly stochastic.
\vspace{-0.4em}
\begin{assumption}[Doubly stochastic transitions]
\label{assump:latent}
  The transition matrix $\mTrans$ is doubly stochastic, and the marginal distribution of the initial hidden state $h_1$ is stationary with respect to $\mTrans$.
\end{assumption}
\vspace{-0.6em}
This assumption guarantees that the stationary distribution of the latent distribution is uniform for any $t$, and the transition matrix for the reversed chain is simply $\mTrans^\top$.
Moreover, this assumption reduces the parameter space and hence will make the non-identifiability results stronger.

We require the following conditions on the parameters for the discrete HMM:
\begin{assumption}[Non-redundancy, discrete HMM]
\label{assump:discrete_rowNE}
    Every row of $\mEmiss$ is non-zero.
\end{assumption}
\vspace{-0.8em}
Assumption \ref{assump:discrete_rowNE} can be interpreted as requiring each token to have a non-zero probability of being observed, which is a mild assumption.
We also require the following non-degeneracy condition:
\vspace{-0.2em}
\begin{assumption}[Non-degeneracy, discrete HMM]
\label{assump:discrete_full_rank}
$\rank(\mTrans) = \rank(\mEmiss) = k \leq d$.
\end{assumption}
\vspace{-0.6em}
Note that Assumption \ref{assump:discrete_full_rank} only requires the parameters to be non-degenerate, rather than have singular values bounded away from 0.
The reason is that this work will focus on population level quantities and make no claims on finite sample behaviors or robustness.

For \GHMM, we similarly require the parameters to be non-degenerate:
\vspace{-0.2em}
\begin{assumption}[Non-degeneracy, \GHMM]
\label{assump:gmm_full_rank}
$\rank(\mTrans) = \rank(\mMeans) = k \leq d$.
\end{assumption}
\vspace{-0.6em}
Moreover, we assume that the norms of the means are known and equal:
\vspace{-0.2em}
\begin{assumption}[Equal norms of the means] 
\label{assump:gmm_equal_norms}
  For each $i \in [k]$, $\mu_i$ is a unit vector.%
  \footnote{Assumption~\ref{assump:gmm_equal_norms} can be changed to $\|\mu_i\|_2=c$ for all $i \in [k]$, for any other fixed number $c>0$.}
%
\end{assumption}
\vspace{-0.8em}
%
Assumptions \ref{assump:latent}-\ref{assump:gmm_full_rank} are fairly standard~\citep[see, e.g.,][]{anandkumar2012multiview}; Assumption \ref{assump:gmm_equal_norms} may be an artifact of our proofs, and it would be interesting to relax in future work.

Our notion of identifiability from a prediction task (or a collection of prediction tasks) will restrict attention to HMMs satisfying Assumptions \ref{assump:latent}, \ref{assump:discrete_rowNE}, \ref{assump:discrete_full_rank} and {\GHMM}s satisfying Assumptions \ref{assump:latent}, \ref{assump:gmm_full_rank}, \ref{assump:gmm_equal_norms}.

\vspace{-0.6em}

\subsection{Uniqueness of tensor rank decompositions}
\label{sec:tensor}
\vspace{-0.4em}

Some of our identifiability results rely on the uniqueness of \emph{tensor rank-1 decompositions}~\citep{hitchcock1927expression}.
An \emph{order-$t$ tensor} (or \emph{$t$-tensor}) is an $t$-way multidimensional array; a matrix is a 2-tensor.
The \textit{tensor rank} of a tensor $\mW$ is the minimum number $\tRank$ such that $\mW$ can be written as a sum of $\tRank$ rank-1 tensors.
That is, if a $t$-tensor $\mW$ has rank-$\tRank$, it means that $\mW = \sum_{i \in [r]} \otimes_{j \in [t]} \mU_i^{(j)}$ for some matrices $\mU^{(j)} \in \R^{n_j \times r}$, where $U_i^{(j)}$ denotes the $i^{\text{th}}$ column of matrix $\mU^{(j)}$.

In this work, we only need to work with 3-tensors of the form $\mW = \sum_{i \in [\tRank]} A_i \otimes B_i \otimes C_i$ for some matrices $\mA \in \R^{n_1 \times \tRank}$, $\mB \in \R^{n_2 \times \tRank}$, $\mC \in \R^{n_3 \times \tRank}$, as 3-tensors will suffice for identifiability in all of our settings of interest.%
\footnote{To apply our results on higher order tensors, one can consider an order-3 slice of the higher order tensor.}
A classic work by \cite{kruskal77} gives a sufficient condition under which $\mA, \mB, \mC$ can be recovered up to column-wise permutation and scaling.
The condition is stated in terms of the \textit{Kruskal rank}, which is the maximum number $\kRank$ such that every $\kRank$ columns of the matrix are linearly independent.
Let $k_A$ denote the Kruskal rank of matrix $\mA$, then:
\begin{proposition}[Kruskal's theorem, \cite{kruskal77}]
\label{prop:kruskal}
  The components $\mA, \mB, \mC$ of a 3-tensor $\mW := \sum_{i\in[\tRank]} \mA_i \otimes \mB_i \otimes \mC_i$ are identifiable up to a shared column-wise permutation and column-wise scaling if $k_A + k_B + k_C \geq 2\tRank + 2$.
\end{proposition}
\vspace{-0.6em}
We note that this work focuses on identifiability results rather than providing an algorithm or sample complexity bounds, though the proofs can be adapted into algorithms \citep[see, e.g.,][]{Jennrich70} under slightly more restrictive conditions (which will be satisfied by all of our identifiability results).

\vspace{-1em}
\section{Identifiability from masked prediction tasks}
\label{sec:results}

\subsection{Pairwise prediction}
\label{s:pairwise} 
We begin with the simplest prediction task: namely predicting one token from another.
We refer to such tasks as \emph{pairwise prediction tasks}.
For HMMs, this task fails to provide parameter identifiability:
\ifx\coltSpace\undefined
\vspace{-2em}
\fi
\begin{theorem}[Nonidentifiability of HMM from predicting $x_2|x_1$]
\label{thm:discrete_2_1_nonId}
    There exists a pair of HMM distributions with parameters $(\mEmiss, \mTrans)$ and $(\tilde\mEmiss, \tilde\mTrans)$, each satisfying Assumptions~\ref{assump:latent}, \ref{assump:discrete_rowNE} and \ref{assump:discrete_full_rank}, and also $\tilde\mEmiss \neq \mEmiss$,
    such that the optimal predictors for the task $x_2|x_1$ are the same under each distribution.
\end{theorem}
\vspace{-1em}
Theorem \ref{thm:discrete_2_1_nonId} is not very surprising: the optimal predictor for the HMM above has the form of a product of (stochastic) matrices, and generally, one cannot uniquely recover matrices from their product sans additional conditions~\citep{donoho2003optimally,candes2006stable,spielman2012exact,arora2014new,georgiev2005sparse,aharon2006uniqueness,cohen2019identifiability}.

On the other hand, pairwise prediction actually \emph{does} suffice for identifiability for \GHMM:
\ifx\coltSpace\undefined
\vspace{-0.8em}
\fi
\begin{theorem}[Identifiability of \GHMM from predicting $x_2|x_1$]
\label{thm:gmm_2_1_id}
    Under Assumption \ref{assump:latent}, \ref{assump:gmm_full_rank}, and \ref{assump:gmm_equal_norms},
    if the optimal predictors for the task $x_2|x_1$ under the \GHMM distributions with parameters
    $(\mMeans, \mTrans)$ and $(\tilde\mMeans, \tilde\mTrans)$ are the same,
    then $(\mMeans, \mTrans) = (\tilde\mMeans, \tilde\mTrans)$
    up to a permutation of the hidden state labels.
\end{theorem}
\vspace{-0.6em}
Comparing Theorem \ref{thm:discrete_2_1_nonId} and \ref{thm:gmm_2_1_id} shows that the specific parametric form of the optimal predictor matters.
Note that HMM and \GHMM have a similar form when conditioning on the latent variable; that is, $P(x_2|h_2=i) = \mEmiss\mTrans_i$ for HMM, and $P(x_2|h_2) = \mMeans\mTrans_i$ for \GHMM.
The salient difference between these two setups lies in the posterior function:
while the posterior function for HMM is linear in the input as we saw earlier, the posterior function for \GHMM is more complicated and ``reveals'' more information about the parameter. 

More precisely, even matching the posterior function \emph{nearly} suffices to identify $\mMeans$:
if $\mMeans, \tilde\mMeans$ parameterize two posterior functions $\phi, \tilde\phi$ where $\phi = \tilde\phi$,
then up to a permutation, $\tilde{\mMeans}$ must be equal to either $\mMeans$ or a unique (and somewhat special) transformation of $\mMeans$. The (special)  transformation can also be excluded, by using the constraint that $T, \tilde{T}$ are stochastic matrices. 

The first claim can be formalized as follows: 
\ifx\coltSpace\undefined
\vspace{-0.8em}
\fi
\begin{lemma}
\label{lem:gmm_phi_reflect}
    For $d \geq k \geq 2$,
    under Assumption \ref{assump:gmm_full_rank}, \ref{assump:gmm_equal_norms},
    $\phi = \tilde\phi$ implies $\tilde\mMeans = \mMeans$ or $\tilde{\mMeans} = \mH \mMeans$,
    where $\mH$ is a Householder transformation of the form $\mH := \mI_d - 2\hat{v}\hat{v}^\top \in \R^{d \times d}$,
    with $\hat{v} := \frac{(\mMeans^{\dagger})^\top\1}{\sqrt{\1^\top\mMeans^{\dagger} (\mMeans^\dagger)^{\top}\1}}$.
\end{lemma}
\vspace{-0.5em}
To provide some geometric intuition about how $\mH$ acts on $\mMeans$,
note that $\hat{v}$ is a unit vector in the column space of $\mMeans$ and perpendicular to the affine hull of $\gA := \{\mu_i: i \in [k]\}$,
which means $\hat{v}^\top\mu_i$ is the same for all $i \in [k]$.
As a result, $\tilde\mMeans = [\tilde\mu_1, ..., \tilde\mu_k] = [\mH\mu_1, ..., \mH\mu_k] = \mMeans - 2(\hat{v}^\top\mu_1)[\hat{v}, ..., \hat{v}]$ is a translation of $\mMeans$ along the direction of $\hat{v}$,
such that the translated points $\{\tilde\mu_i\}_{i\in[k]}$ lie on 
the opposite side of the origin.
It is non-trivial to argue that $\mH\mMeans$ is the only solution (other than $\mMeans$ itself) that preserves $\phi$---the proof is  deferred to Appendix \ref{sec:proof_gmm_phi_reflect}.
It is, however, easy to see that $\mH\mMeans$ indeed results in a matching posterior: for this to happen, it's sufficient that 1) $\tilde\mMeans$ is a translation of $\mMeans$, and 2) $\|\tilde\mu_i\|^2 - \|\mu_i\|^2$ is the same for all $i \in [k]$. A quick calculation shows that
$\tilde\mMeans := \mH\mMeans$ indeed satisfies both of these conditions.


\textit{Proof sketch for Theorem~\ref{thm:gmm_2_1_id}}:
We first show that if $\mMeans,\mTrans$ and $\tilde\mMeans,\tilde\mTrans$ produce the same predictor,
then their posterior function must be equal up to a permutation (Lemma~\ref{lem:gmm_matchPhi}).
We can then apply Lemma \ref{lem:gmm_phi_reflect} to recover $\mMeans$ up to a permutation and a Householder transformation $\mH$.
Then, we show that if $\tilde\mMeans = \mH\mMeans$, then the corresponding $\tilde\mTrans$ must have negative entries and thus would not be a valid stochastic matrix.
Hence it must be that $\tilde\mMeans, \mMeans$ are equal up to permutation.

Finally, by way of remarks, another way to think of the difference between the two setups is that for HMM, $P(x_2|x_1)$ is a mixture of categorical distributions, which itself is also a categorical distribution.
This also implies that the nonidentifiability from pairwise prediction in the HMM case cannot be resolved by changing the squared loss to another proper loss function.
On the other hand, for \GHMM, the conditional distribution $P(x_2|x_1)$ is a mixture of Gaussians, which is well known to be identifiable. In fact, if we were given access to the \emph{entire} conditional distribution $P(x_2|x_1)$ (instead of just the conditional mean), it's even easier to prove identifiability for \GHMM. Though this is already implied from identifiability from the conditional means, we provided a (much simpler) proof in Appendix \ref{appendix:gmm_cond_distr} assuming access to the full conditional distribution.

\subsection{Beyond pairwise prediction}
\label{s:threewise}

The conclusion from Theorem \ref{thm:discrete_2_1_nonId} is that a single pairwise prediction task does not suffice for identifiability on HMMs.
The next natural question is then: can we modify the task to obtain identifiability.
A natural idea is to force the model to ``predict more'', and one straightforward way to do so is to train on more than one pairwise prediction tasks.
Unfortunately, this does not resolve the problem:
\ifx\coltSpace\undefined
\vspace{-0.8em}
\fi
\begin{theorem}[Nonidentifiability of HMM from multiple pairwise predictions]
\label{thm:discrete_nonId}
    There exists a pair of HMM distributions with parameters $(\mEmiss, \mTrans)$ and $(\tilde\mEmiss, \tilde\mTrans)$, each satisfying Assumptions~\ref{assump:latent}, \ref{assump:discrete_rowNE} and \ref{assump:discrete_full_rank},
    and also $\tilde\mEmiss \neq \mEmiss$,
     such that, for each of the tasks $x_2|x_1$, $x_1|x_2$, $x_3|x_1$, and $x_1|x_3$,
     the optimal predictors are the same under each distribution.%
    \footnote{Theorem \ref{thm:discrete_nonId} subsumes Theorem \ref{thm:discrete_2_1_nonId} as a special case.}
\end{theorem}
\ifx\coltSpace\undefined
\vspace{-0.8em}
\fi
Recall that the limitation of pairwise predictions on HMMs comes from non-uniqueness of matrix factorization.
While adding additional pairwise prediction tasks introduces more equations on the product of matrices, these equations are highly dependent---and we provide examples that can simultaneously satisfy all these equations.

Another way of forcing the model to ``predict more'' is to consider predictions of joint statistics involving multiple tokens. The hope is that doing so results in equations on tensors---as opposed to matrices--- for which there is a lot of classical machinery delineating tensors for which the rank-1 decomposition is unique, as discussed in Section~\ref{sec:tensor}.
This intuition proves to be true and gives the following result:
\ifx\coltSpace\undefined
\vspace{-0.6em}
\fi
\begin{theorem}[Identifiability from masked prediction on three tokens, HMM]
\label{thm:discrete}
    Let $(\tOne, \tTwo, \tThr)$ be any permutation of $(1,2,3)$, and consider the prediction task $x_{\tTwo} \otimes x_\tThr | x_\tOne$.
    Under Assumption \ref{assump:latent}, \ref{assump:discrete_rowNE}, \ref{assump:discrete_full_rank},
    if the optimal predictors under the HMM distributions with parameters $(\mEmiss,\mTrans)$ and  $(\tilde\mEmiss, \tilde\mTrans)$ are the same,
    then $(\mEmiss,\mTrans) = (\tilde\mEmiss, \tilde\mTrans)$
    up to a permutation of the hidden state labels.
\end{theorem}
\vspace{-0.6em}
The difficulty compared to prior results on identifiability from third order moments ~\citep{allman2009identifiability,anandkumar2012multiview,anandkumar2014tensor} is that we only have access to the conditional 2-tensors (i.e. matrices) given by the predictors.  
The proof idea is to construct a third-order tensor by linearly combining the conditional 2-tensors for each possible value of the token being conditioned on, such that Kruskal's theorem can be applied to obtain identifiability. Note, importantly, that the weights for the linear combination cannot depend on the marginal probabilities of the token being conditioned on, as it is unclear whether there could be multiple singleton marginals consistent with the conditional probabilities we are predicting. Thus, the above theorem cannot be simply derived from results showing parameter identifiability from the 3rd order moments.

Using Lemma \ref{lem:gmm_phi_reflect},
this tensor decomposition argument can also be applied to \GHMM:
\ifx\coltSpace\undefined
\vspace{-0.6em}
\fi
\begin{theorem}[Identifiability from masked prediction on three tokens, \GHMM]
\label{thm:gmm}
    Let $(\tOne, \tTwo, \tThr)$ be any permutation of $(1,2,3)$, and consider the prediction task $x_{\tTwo} \otimes x_\tThr | x_\tOne$.
    Under Assumption \ref{assump:latent}, \ref{assump:gmm_full_rank}, \ref{assump:gmm_equal_norms},
    if the optimal predictors under the $\GHMM$ distributions with parameters $(\mMeans,\mTrans)$ and $(\tilde\mMeans, \tilde\mTrans)$ are the same, then $(\mMeans,\mTrans) = (\tilde\mMeans, \tilde\mTrans)$
    up to a permutation of the hidden state labels.
\end{theorem}
\vspace{-0.5em}
We briefly remark that
the reason we only consider adjacent time steps is that when all tokens involved in the prediction task are at least two time steps apart,
matching predictors only matches powers (of more than one) of the transition matrices---which in general doesn't suffice to ensure the transition matrices themselves are matched.

\vspace{-0.6em}

\section{Proofs}
\label{sec:proofs}

This section includes proofs of some In Section \ref{sec:proof_hmm_id} we prove the identifiability of HMM parameters from the task of predicting two tokens (Theorem \ref{thm:discrete}) using ideas from tensor decomposition.
Then, Section \ref{sec:proof_hmm_nonId} provides an example showing the nonidentifiability of HMM from multiple pairwise predictions (Theorem \ref{thm:discrete_nonId}).
The identifiability of pairwise prediction on \GHMM~is proved in Section \ref{sec:proof_gmm_2_1_id},
and the rest of the proofs are deferred to the appendix.
\vspace{-1em}


\subsection{Proof of Theorem \ref{thm:discrete}: identifiability of predicting two tokens for HMM}
\label{sec:proof_hmm_id}




There are three cases for the two-token prediction task, i.e. 
    1) $x_2 \otimes x_3 | x_1$, 
    2) $x_1 \otimes x_3 | x_2$,
    and 3) $x_1 \otimes x_2 | x_3$.
We will prove for the first two cases, as the third case is proved the same way as the first case by symmetry.
In all cases, the idea is to use the predictor to construct a 3-tensor whose components are each of rank-$k$,
so that applying Kruskal's theorem gives identifiability.

\vspace{-0.6em}
\paragraph{Case 1, $x_2 \otimes x_3|x_1$:}
%
$\mEmiss,\mTrans$ and $\tilde\mEmiss, \tilde\mTrans$ producing the same predictor means $f^{2\otimes 3|1}(x_1) := \E[x_2 \otimes x_3|x_1] = \tilde{\E}[x_2 \otimes x_3|x_1] := \tilde{f}^{2\otimes 3|1}(x_1)$,
where $\E, \tilde\E$ are parameterized by the corresponding parameters.
Let $\gX := \{\ei: i\in[d]\}$, and consider the following 3-tensor:
\begin{equation}\label{eq:hmm_23_1_decomp}
\begin{split}
    \Tensor
    :=& \sum_{x_1 \in \gX} x_1 \otimes \E[x_2 \otimes x_3|x_1]
    = \sum_{x_1 \in \gX} x_1 \otimes \E_{h_2|x_1}[\E[x_2 \otimes x_3|x_1] | h_2]
    \\
    =& \sum_{x_1 \in \gX} x_1 \otimes \sum_{h_2} P(h_2|x_1) \E[x_2|h_2] \otimes \E[x_3|h_2]
    \\
    =& \sum_{i\in[k]} \sum_{x_1 \in \gX} P(h_2=i|x_1) x_1 \otimes \E[x_2|h_2=i] \otimes \E[x_3|h_2=i]
    \\
    =& \sum_{i\in[k]} \Big(\underbrace{\sum_{x_1 \in \gX} (\mTrans\phi(x_1))^\top \ei^{(k)} x_1}_{:= a_i} \Big) \otimes \mEmiss_i \otimes (\mEmiss\mTrans)_i,
    \\
\end{split}
\end{equation}
where $\mEmiss_i$ denotes the $i^{\text{th}}$ column of $\mEmiss$, and similarly for $(\mEmiss\mTrans)_i$.
Note that $\Tensor$ can also be written as 
\begin{equation}
\begin{split}
    \Tensor 
    =& \sum_{x_1 \in \gX} x_1 \otimes \tilde\E[x_2 \otimes x_3|x_1]
    = \sum_{i\in[k]} \Big(\sum_{x_1 \in \gX} (\tilde\mTrans \tilde\phi(x_1))^\top \ei^{(k)} x_1 \Big) \otimes \tilde\mEmiss_i \otimes (\tilde\mEmiss\tilde\mTrans)_i.
    \\ 
\end{split}
\end{equation}
We would like to apply Kruskal's theorem for identifiability.
In particular, we will show that each component in equation \ref{eq:hmm_23_1_decomp} forms a matrix of Kruskal rank $k$.
The second and third components clearly satisfy this condition by Assumption \ref{assump:discrete_full_rank}.
For the first component, write $a_i$ as
\vspace{-0.8em}
\begin{equation}
\begin{split}
    a_i
    =& \sum_{j\in [d]} \big(\mTrans\phi(\ej^{(d)})\big)^\top \ei^{(k)} \cdot \ej^{(d)}
    = \sum_{j \in [d]} \frac{(\ej^{(d)})^\top \mEmiss}{\|(\ej^{(d)})^\top\mEmiss\|_1} \mTrans^\top \ei^{(k)} \cdot \ej^{(d)}
    \\
    =& \diag\big([1/\|(\ej^{(d)})^\top\mEmiss\|_1]_{j \in [d]}\big) \mEmiss \mTrans^\top \ei^{(k)}.
    \\
\end{split}
\end{equation}
Putting $a_i$ into a matrix form, we get $\mA := [a_1, ..., a_k] = \diag\big([1/\|(\ej^{(d)})^\top\mEmiss\|_1]_{j \in [d]}\big) \mEmiss \mTrans^\top$,
\footnote{We use $[\alpha_i]_{i \in [d]}$ to denote a $d$-dimensional vector whose $i^{\text{th}}$ entry is $\alpha_i$.}
which is of rank $k$ by Assumption \ref{assump:discrete_full_rank}.
Hence components $\Tensor$ are all of Kruskal rank $k$, and columns of $\mEmiss\mTrans, \mEmiss$ are identified up to column-wise permutation and scaling by Kruskal's theorem.
The indeterminacy in scaling is further removed noting that columns of $\mEmiss, \mTrans$ need to sum up to 1.
Lastly, $\mTrans$ is recovered as $\mTrans = \mEmiss^\dagger\mEmiss\mTrans$.


%

\paragraph{Case 2, $x_1 \otimes x_3|x_2$:}
The optimal predictor for the task of predicting $x_1, x_3$ given $x_2$ takes the form
\begin{equation}
\begin{split}
    \E[x_1\otimes x_3|x_2] = (\mEmiss\mTrans^\top) \diag(\phi(x_2)) (\mEmiss\mTrans)^\top.
\end{split}
\end{equation}
Similarly as the previous case, we would like to construct a 3-tensor whose components can be uniquely determined by Kruskal's theorem.
Let $\gX$ be the same as before, and consider the 3-tensor
\begin{equation}
\begin{split}
    \Tensor
    :=& \sum_{x_2 \in \gX} x_2 \otimes \E[x_1\otimes x_3|x_2]
    = \sum_{x_2 \in \gX} x_2 \otimes \E_{h_2|x_2} (\E[x_1|h_2] \otimes \E[x_3|h_2])
    \\
    =& \sum_{i \in [k]} \underbrace{\sum_{x_2 \in \gX} (\phi(x_2))^\top \ei^{(k)} x_2}_{:= a_i} \otimes \E[x_1|h_2] \otimes \E[x_3|h_2]
    = \sum_{i \in [k]} a_i \otimes (\mEmiss\mTrans^\top)_i \otimes (\mEmiss\mTrans)_i,
\end{split}
\end{equation}
where the first component can be simplified to
\begin{equation}
\begin{split}
    a_i
    =& \sum_{j\in[d]} \frac{(\ej^{(d)})^\top \mEmiss}{\|(\ej^{(d)})^\top\mEmiss\|_1} \ei^{(k)} \cdot \ej^{(d)}
    = \Big(\diag\big([\|\mEmiss_j^\top\|_1]_{j \in [d]}\big)\Big)^{-1} \mEmiss \ei^{(k)}
    := \mD^{-1} \mEmiss \ei^{(k)}.
\end{split}
\end{equation}
The matrix $\mA := [a_1, ..., a_k] = \mD^{-1}\mEmiss$ is of rank $k$, hence we can identify (up to permutation) columns of each component of $\Tensor$ by Kruskal's theorem.
This means if $\mEmiss,\mTrans$ and $\tilde\mEmiss, \tilde\mTrans$ produce the same predictor,
then we have $\mEmiss\mTrans = \tilde\mEmiss\tilde\mTrans$, $\mEmiss\mTrans^\top = \tilde\mEmiss\tilde\mTrans^\top$,
and that $\mEmiss, \tilde\mEmiss$ are matched up to a scaling of rows (i.e. $\mD^{-1}$).
%
Next, to determine $\mD$, note that $\mTrans,\tilde\mTrans$ are doubly stochastic by Assumption \ref{assump:latent},
which means the all-one vector $\1 \in \R^{k}$ satisfies $\mTrans\1 = \tilde\mTrans\1 = \1$.
Hence $\tilde\mEmiss\tilde\mTrans\1 = \mEmiss\mTrans\1 = \mEmiss\1 = [\|\mEmiss_{j}^\top\|_1]_{j \in [d]}$.
We can then compute $\mD$ as $\mD = \diag(\mEmiss\mTrans\1)$,
and recover $\mEmiss$ as $\mEmiss = \mD\mA$.
Finally, $\mTrans$ is also recovered since
$\tilde\mEmiss\tilde\mTrans = \mEmiss\tilde\mTrans
= \mEmiss\mTrans
\Rightarrow \tilde\mTrans \mTrans^{-1} = \mI_k
\Rightarrow \tilde\mTrans = \mTrans$.


\subsection{Proof of Theorem \ref{thm:discrete_nonId}: non-identifiability of pairwise prediction tasks on HMM}
\label{sec:proof_hmm_nonId}

We provide an example to show the nonidentifiability result in Theorem \ref{thm:discrete_nonId}, which can also serve as an example for Theorem \ref{thm:discrete_2_1_nonId}.
The goal is to find $\tilde\mEmiss \neq \mEmiss$, $\tilde\mTrans \neq \mTrans$ that produce the same predictors for predicting both $x_2|x_1$ and $x_3|x_1$.
We will choose $\mTrans, \tilde\mTrans$ to be symmetric, so that $\mEmiss,\mTrans$ and $\tilde\mEmiss, \tilde\mTrans$ also form the same predictors for the reversed direction, i.e., for predicting $x_1$ given $x_2$ and $x_1$ given $x_3$, since the reverse chain has transition matrix $\mTrans^\top = \mTrans$.

Let's consider the case where
the all row sums of $\mEmiss$ and $\tilde\mEmiss$ are $k/d$.
Consequently, the posterior function is simply $\phi(x) = \frac{\mEmiss^\top x}{\|\mEmiss^\top x\|_1} = \frac{d}{k} \mEmiss^\top x$, and similarly we have $\tilde\phi(x) = \frac{d}{k}\tilde{\mEmiss}^\top x$.
The predictors are of the form:
\begin{equation}
\begin{split}
  f^{2|1}(x) = \mEmiss\mTrans\phi(x) = \frac{d}{k}\mEmiss\mTrans\mEmiss^\top x,~
  f^{3|1}(x) = \mEmiss\mTrans^2\phi(x) = \frac{d}{k}\mEmiss\mTrans^2 \mEmiss^\top x.
\end{split}
\end{equation}
Matching $f^{2|1}(x) = \tilde{f}^{2|1}(x)$ on all $x \in \gX := \{\ei\}_{i \in [d]}$ means
\begin{equation}\label{eq:31_21_part1}
\begin{split}
  &\mEmiss\mTrans\mEmiss^\top \mI_d = \mEmiss\mTrans\mEmiss^\top
  = \tilde\mEmiss \tilde\mTrans \tilde\mEmiss^\top 
  \Rightarrow \tilde\mTrans = \tilde\mEmiss^\dagger \mEmiss \cdot \mTrans \cdot (\tilde\mEmiss^\dagger \mEmiss)^{\top}.
  \\
\end{split}
\end{equation}
Similarly, matching $f^{3|1} = \tilde{f}^{3|1}$ gives $\mEmiss\mTrans^2\mEmiss^\top = \tilde\mEmiss \tilde\mTrans^2 \tilde\mEmiss^\top$, hence 
\begin{equation}\label{eq:discrete_nonid_31}
\begin{split}
  \tilde\mEmiss \tilde\mTrans^2 \tilde\mEmiss^\top
  =&~ \tilde{\mEmiss} \tilde\mEmiss^\dagger \mEmiss \mTrans (\tilde\mEmiss^\dagger \mEmiss)^\top \cdot \tilde\mEmiss^\dagger \mEmiss \mTrans (\tilde\mEmiss^\dagger \mEmiss)^\top \tilde\mEmiss^\top
  \\
  \eqLabel{1}=&~ \mEmiss \mTrans \cdot (\tilde\mEmiss^\dagger \mEmiss)^\top \tilde\mEmiss^\dagger \mEmiss \cdot \mTrans\mEmiss^\top 
  = \mEmiss \mTrans \cdot \mTrans\mEmiss
  \Rightarrow (\tilde\mEmiss^\dagger \mEmiss)^\top \cdot \tilde\mEmiss^\dagger \mEmiss = \mI_k,
\end{split}
\end{equation}
where step $\rnum{1}$ uses $\tilde\mEmiss\tilde\mEmiss^\dagger \mEmiss = \mEmiss$, since $\tilde\mEmiss, \mEmiss$ share the same column space.

Denote $\mR := \tilde\mEmiss^\dagger \mEmiss$; $\mR$ is orthogonal by the last equality in equation \ref{eq:discrete_nonid_31}.
To construct the desired example, consider $k=3$, and let $\mR$ represent a rotation with axis of rotation $\frac{1}{3}(e_1 + e_2 + e_3)$.
This axis is the direction pointing from the origin to the projection of the origin on the hyperplane $\gP_c := \{v \in \R^d: \sum_{i\in[d]} v_i = c\}$ for any positive constant $c$ (i.e. $\gP_c$ is parallel to the hyperplane in which probability vectors lie).
This means such rotation guarantees $\mR v \in \gP_c$, $\forall v \in \gP_c$, and has the property (proved in Appendix \ref{appendix:hmm_claim_rotSum}) that
\ifx\coltSpace\undefined
\vspace{-0.6em}
\fi
\begin{claim}
\label{clm:rotation_sumRowCol}
    Each row and each column of $\mR$ sums up to 1.
\end{claim}
\vspace{-0.8em}
Define $\tilde\mEmiss := \mEmiss\mR$, $\tilde\mTrans := \mR^\top \mTrans \mR$, Claim \ref{clm:rotation_sumRowCol} ensures that row sum and column sum of $\tilde\mEmiss, \tilde\mTrans$ remain the same as those of $\mEmiss, \mTrans$.
When the rotation angle represented by $\mR$ is sufficiently small, entries $\tilde\mEmiss, \tilde\mTrans$ remain in $[0,1]$, hence such $\tilde\mEmiss, \tilde\mTrans$ form a valid example.
We provide a concrete example in Appendix \ref{appendix:hmm_eg_12_21_13_31}.

\subsection{Proof of Theorem \ref{thm:gmm_2_1_id}: identifiability of predicting $x_2$ given $x_1$ for \GHMM}
\label{sec:proof_gmm_2_1_id}


%
    For \GHMM, the predictor for $x_2$ given $x_1$ is parameterized as $f^{2|1}(x_1) = \E[x_2|x_1] = \mMeans\mTrans \phi(x_1)$.
    If $\mMeans,\mTrans$ and $\tilde\mMeans, \tilde\mTrans$ produce the same predictor, then 
    \begin{equation}
    \begin{split}
        f^{2|1}(x) = \mMeans\mTrans \phi(x)
        = \tilde\mMeans \tilde\mTrans \tilde\phi(x)
        = \tilde{f}^{2|1}(x), ~\forall x \in \R^d.
    \end{split}
    \end{equation}
    Let $\mR := (\tilde\mMeans\tilde\mTrans)^\dagger (\mMeans\mTrans) \in \R^{k \times k}$, then $\tilde\phi(x) = \mR\phi(x)$.
    The following lemma (proof deferred to Appendix \ref{sec:gmm_helper_lem}) says that $\phi, \tilde\phi$ must then be equal up to a permutation of coordinates:
    \vspace{-0.8em}
    \begin{lemma}
    \label{lem:gmm_matchPhi}
        If there exists a non-singular matrix $\mR \in \R^{k \times k}$ such that 
        $\phi(x) = \mR\tilde\phi(x)$, $\forall x \in \R^d$, then $\mR$ must be a permutation matrix.
    \end{lemma}
    \vspace{-0.8em}
    Combined with Lemma \ref{lem:gmm_phi_reflect}, we have $\tilde\mMeans$ is equal to (up to a permutation) either $\mMeans$ or $\mH\mMeans$, where $\mH$ is the Householder reflection given in Lemma \ref{lem:gmm_phi_reflect}.
    
    The remaining step is to show that $\mH\mMeans$ can be ruled out by requiring $\tilde\mTrans$ to be a stochastic matrix.
    Note that matching both the predictor and the posterior function means we also have $\tilde\mMeans\tilde\mTrans = \mMeans\mTrans$, or $\tilde\mTrans = (\tilde\mMeans^{\dagger}\mMeans) \mTrans$.
    Recall that $\mH := \mI_d - 2\hat{v}\hat{v}^\top$
    for $\hat{v} = \frac{(\mMeans^{\dagger})^\top\1}{\sqrt{\1^\top\mMeans^{\dagger}(\mMeans^\dagger){\top}\1}}$.
    When $\tilde\mMeans = \mH\tilde\mMeans$, the column sum of $\tilde\mMeans^{\dagger}\mMeans$ is
    \begin{equation}
    \begin{split}
        \1^\top \tilde\mMeans^{\dagger} \mMeans
        =& \1^\top \mMeans^{\dagger}\mH^{-1}\mMeans
        = \1^\top \mMeans^{\dagger} (\mI - 2\hat{v}\hat{v}^\top) \mMeans
        \\
        =& \1^\top(\mI - 2\mMeans^{\dagger}\hat{v}\hat{v}^\top \mMeans)
        = \1^\top - 2 \cdot \1^\top \frac{\mMeans^{\dagger} (\mMeans^\dagger)^{\top} \1 \1^\top \mMeans^{\dagger} \mMeans}{\1^\top\mMeans^{\dagger}(\mMeans^\dagger)^{\top}\1} 
        \\
        =& \1^\top - 2\cdot \frac{\1^\top \mMeans^{\dagger} (\mMeans^\dagger)^{\top} \1}{\1^\top\mMeans^{\dagger} (\mMeans^\dagger)^{\top}\1} \1^\top 
        = \1^\top - 2\cdot \1^\top
        = -\1^\top.
    \end{split}
    \end{equation}
    This means the column sum of $\tilde\mTrans$ is 
    $\1^\top \tilde\mTrans
    = \1^\top (\tilde\mMeans^{\dagger}\mMeans) \mTrans
    = -\1^\top \mTrans
    = -\1^\top$,
    which violates the constraint that $\tilde\mTrans$ should be a stochastic matrix with positive entries and column sum 1.
    Hence it must be that $\mMeans = \tilde\mMeans$ and hence also $\mTrans = \tilde\mTrans$ (up to permutation), proving the theorem statement.

%

\section{Related works}


\paragraph{Self-supervised learning}
On the empirical side, self-supervised methods have gained a great amount of popularity across many domains,
including language understanding~\citep{word2vec,vaswani2017attention,bert}, 
visual understanding~\citep{relPos15,inpaint16},
and distribution learning~\citep{nce,gao2020flow}.
Classic ideas such as contrastive and masked prediction remain powerful in their modern realizations~\citep{CPCv2,simclr,iGPT,MAE21}, pushing the state of the art performance and even surpassing supervised pretraining in various aspects~\citep{lee2021compressive,liu2021robust}.

On the theoretical front, there have been analyses on both masked predictions~\citep{predictLee20} and contrastive methods~\citep{arora2019theoretical,multiViewTosh20,topicTosh21,WangIsola20,haochen2021provable,Zixin20} though with a focus on characterizing the quality of the learned features for downstream tasks \citep{Saunshi20LM,wei2021pretrained}. These approaches usually rely on quite strong assumptions to tie the self-supervised learning objective to the downstream tasks of interest.
However, our results show that the specific parametric form is key to identifiability.
While the parameter recovery lens is a new contribution of our work, \cite{Zixin20} argue (as a side-product of their analysis) that some aspects of a generative model are recovered in their setup.
Their data model, however, is substantially different from ours and has very different identifiability properties (owing to its basis in sparse coding).

\paragraph{Latent variable models and tensor methods}
Latent variable models have been widely studied in the literature.
One important area of research is independent component analysis (ICA),
where the data are assumed to be given as a transformation (mixing) of unknown independent sources which ICA methods aim to identify. In nonlinear ICA data models, both the sources and the mixing function are generally not identifiable.  \cite{TCL,PCL} however, under some additional assumptions
(on the dependency structure of the different time steps, among other things),
provide some identifiability results on the sources. Unlike our setup, the mixing function in these models is deterministic and also not the object of recovery.

More related to this work is the line of work on learning latent variable models with tensor methods.
Specific to learning HMMs, \citet{mossel2005learning} and \citet{anandkumar2012multiview,anandkumar2014tensor} provide algorithms based on third-order moments.
A major difference between these prior works on tensor methods and ours is that previous results operate on joint moments, while the results in this work are based on conditional moments given by the optimal predictors for the masked tokens.

\section{Conclusion}


In this work, we initiated a study of masked prediction, an empirically successful self-supervised learning method, through the lens parameter recovery.
By studying various prediction tasks on HMM and a conditionally-Gaussian variant (\GHMM), we showed that parameter recovery can be very sensitive to the interaction between the data model, prediction task, and concomitant predictors.
We hope this work inspires further research, as many tantalizing open questions remain. 
For instance, how does changing the number of either predicted or conditioned-on tokens affect identifiability? Are certain tasks ``better'' than others---e.g., in terms of robustness or sample complexity? Can one derive similar results for more general families of HMMs?


\bibliography{references.bib}
\bibliographystyle{abbrvnat}

\newpage

\appendix

\section{Missing proofs for \GHMM}
\label{appendix:gmm}

This section provides missing proofs for results on \GHMM.
We will first prove the two lemmas on properties of the posterior function (Lemma \ref{lem:gmm_phi_reflect}, \ref{lem:gmm_matchPhi}),
then show the proof for the three-token prediction task (Theorem \ref{thm:gmm}) using a tensor decomposition idea similar to that of Theorem \ref{thm:discrete}.
At the end, we show the identifiability from pairwise conditional distributions (as opposed to conditional expectation as in masked prediction tasks),
which is proved by reducing parameter recovery to the identifiability of Gaussian mixtures (Theorem \ref{thm:gmm_conditional}).

\subsection{Proofs of helper lemmas}
\label{sec:gmm_helper_lem}

\subsubsection{Proof for Lemma \ref{lem:gmm_matchPhi}}
\label{sec:proof_gmm_matchPhi}

\begin{lemma*}[Lemma \ref{lem:gmm_matchPhi} restated]
    If there exists a non-singular matrix $\mR \in \R^{k \times k}$ such that 
    $\phi(x) = \mR\tilde\phi(x)$, $\forall x \in \R^d$, then $\mR$ must be a permutation matrix.
\end{lemma*}

\begin{proof}
    We will prove the lemma by matching the Jacobian w.r.t. $x$ on both sides.
    Let's first quickly recall the Jacobian of the posterior vector $\phi(x) \in \R^k$, where $[\phi(x)]_i = \frac{\exp(-\frac{\|x-\mu_i\|^2}{2})}{\sum_{j \in [k]} \exp(-\frac{\|x-\mu_j\|^2}{2})}$.
    Denote $o(x) := \big[-\frac{\|x-\mu_1\|^2}{2}, ..., -\frac{\|x-\mu_k\|^2}{2}\big] \in \R^k$,
    then $\nabla_x \phi(x) = \nabla_{o(x)} \softmax(o(x)) \cdot \nabla_{x} o(x)$,
    where 
    \begin{equation}
    \begin{split}
        \nabla_o [\softmax(o)]_i
            =&~ [\softmax(o)]_i \cdot (\ei - \softmax(o))
            = [\phi(x)]_i \cdot (\ei - \phi(x)),
        \\
        \nabla_o \softmax(o)
            =&~ \diag(\phi(x)) - \phi(x)\phi(x)^\top,
        \\
        \nabla_x o(x) =& -[x - \mu_1, ..., x - \mu_k]^\top.
    \end{split}
    \end{equation}
    Hence the Jacobian is 
    \begin{equation}
        \nabla_x \phi(x) = \big(\diag(\phi(x)) - \phi(x)\phi(x)^\top \big) \cdot (\mMeans - [x, x,..., x])^\top.
    \end{equation}
    Denote $\Delta := \mMeans - [x, x, ..., x] \in \R^{d \times k}$, and similarly $\tilde\Delta = \tilde\mMeans - [x,x,...,x]$.
    Matching $\nabla_x \tilde\phi(x) = \nabla_x \mR\phi(x)$ gives
    \begin{equation}\label{eq:gmm_2_1_matchPhi_grad}
    \begin{split}
        \diag(\mR\phi(x)) \tilde\Delta^\top - \mR\phi(x) (\tilde\Delta\mR\phi(x))^\top 
            = \mR \diag(\phi(x)) \Delta^\top - \mR \phi(x) (\Delta\phi(x))^\top.
    \end{split}
    \end{equation}

    Let's take $x = x_c^{(i)} := c\mu_i$ for $c > 1$.
    We claim that this $x_c^{(i)}$ satisfies $\lim_{c\rightarrow \infty}\phi(x_c^{(i)}) \rightarrow \ei$.
    This is because $\forall j \neq i$, 
    \begin{equation}
    \begin{split}
        &\lim_{c \rightarrow \infty}\frac{[\phi(x_c^{(i)})]_j}{[\phi(x_c^{(i)})]_i}
        = \lim_{c \rightarrow \infty}\exp\Big(\frac{\|c\mu_i - \mu_i\|^2}{2} - \frac{\|c\mu_i - \mu_j\|^2}{2}\Big)
        \\
        =& \lim_{c \rightarrow \infty}\exp\Big(-\frac{\big((2c-1)\mu_i - \mu_j\big)^\top (\mu_i - \mu_j)}{2}\Big)
        = \lim_{c \rightarrow \infty}\exp\Big(-\frac{2c\mu_i^\top (\mu_i - \mu_j)}{2}\Big)
        = 0
    \end{split}
    \end{equation}
    where the last equality is because $\mu_i^\top(\mu_i - \mu_j) > 0$ for any $\mu_i, \mu_j$ lying on the same hypersphere.
    
    With such choices of $x$, the two sides of equation \ref{eq:gmm_2_1_matchPhi_grad} are now:
    \begin{equation}
    \begin{split}
        LHS =&~ \diag(\mR_i)\tilde\Delta^\top - \mR_i\mR_i^\top \tilde\Delta^\top
        = (\diag(\mR_i) - \mR_i\mR_i^\top) \tilde\Delta^\top 
        \\
        = RHS =& \sum_{j\in[k]} [\ei]_j \mR_j (\Delta_j)^\top - \mR\ei (\Delta\ei)^\top
        = \mR_i (\Delta_i)^\top - \mR_i (\Delta_i)^\top = 0.
        \\
    \end{split}
    \end{equation}
    
    Since $x := c\mu_i$ for $c\rightarrow \infty$ lies outside the affine hull of $\{\tilde\mu_i\}_{i\in[k]}$,
    $\tilde{\Delta}$ is of full rank due to the following claim:
    \ifx\coltSpace\undefined
    \vspace{-1em}
    \fi
    \begin{claim}
    \label{clm:affine_comb}
        Given a linearly independent set $\{u_i\}_{i \in [k]}$,
        if $\{u_i - v\}_{i \in[k]}$ is not linearly independent,
        then $v = \sum_{i\in [k]}\beta_i \cdot u_i$ where $\sum_{i \in [k]}\beta_i = 1$.
    \end{claim}
    \begin{proof}
    Since $\{u_i - v\}_{i \in[k]}$ is linearly dependent, we can write some $u_j - v$ as the linear combination of other $\{u_i - v\}_{i \in [k], i \neq j}$.
    Let's take $j = k$ wlog, and denote the coefficients of the linear combination as $\{\alpha_i\}_{i \in [k-1]}$.
    Then 
    \begin{equation}
    \begin{split}
        &u_k - v = \sum_{i \in [k-1]} \alpha_i (u_i - v)
        \Rightarrow \big(1 - \sum_{i \in [k-1]}\alpha_i \big) v = -\sum_{i\in[k-1]} \alpha_i \cdot u_i + u_k
        \\
    \end{split}
    \end{equation}
    The right hand side is non-zero since $\{u_i\}_{i \in [k]}$ are linearly independent by assumption,
    hence $1 - \sum_{i \in [k-1]}\alpha_i \neq 0$, and we get
    \begin{equation}
    \begin{split}
        v = \sum_{i\in[k-1]} \underbrace{\frac{-\alpha_i}{1 - \sum_{i\in[k-1]} \alpha_i}}_{:= \beta_i} \cdot u_i + \underbrace{\frac{1}{1 - \sum_{i \in [k-1]}\alpha_i}}_{:= \beta_k} u_k.
    \end{split}
    \end{equation}

    Note that $\sum_{i \in [k]}\beta_i = 1$, hence $v$ is an affine combination of $\{u_i: i\in[k]\}$.
    \end{proof}

    Since $\tilde\Delta$ is full rank, it must be $\diag(\mR_i) - \mR_i\mR_i^\top = 0$, which implies $\mR$ is a permutation matrix.
    This is because for any non-zero $v$ s.t. $\diag(v) - v v^\top = 0$, the entries of $v$ satisfy $v_i^2=1$, $v_iv_j = 0$ for $i\neq j$.
    Hence $v$ has exactly one non-zero entry which is $\pm 1$.
    Since $\mR\phi(x) = \tilde\phi(x)$ where $\phi(x), \tilde\phi(x)$ are both probability vectors with non-negative entries, this non-zero entry has to be 1 (and not -1).
    Since $\mR$ is of rank-$k$ by Assumption \ref{assump:gmm_full_rank}, this non-zero entry is at different positions for different $\mR_i$, hence $\mR$ is a permutation matrix.


\end{proof}

\subsubsection{Proof of Lemma \ref{lem:gmm_phi_reflect}}
\label{sec:proof_gmm_phi_reflect}

\begin{lemma*}[Lemma \ref{lem:gmm_phi_reflect} restated]
    For $d \geq k \geq 2$,
    then under Assumption \ref{assump:gmm_full_rank}, \ref{assump:gmm_equal_norms},
    $\phi = \tilde\phi$ implies $\tilde\mMeans = \mMeans$ or $\tilde{\mMeans} = \mH \mMeans$,
    where $\mH$ is a Householder transformation of the form $\mH := \mI_d - 2\hat{v}\hat{v}^\top \in \R^{d \times d}$,
    with $\hat{v} := \frac{(\mMeans^{\dagger})^\top\1}{\sqrt{\1^\top\mMeans^{\dagger} (\mMeans^\dagger)^{\top}\1}}$.
\end{lemma*}
\begin{proof}
    Let's start with $d = k$.
    First, let's check the conditions for $\phi = \tilde\phi$.
    For any $x \in \R^d$, we have 
    \begin{equation}
    \begin{split}
        &[\phi(x)]_i = \frac{\exp\big(-\frac{\|x-\mu_i\|^2}{2}\big)}{\sum_{j \in [k]} \exp\big(-\frac{\|x-\mu_j\|^2}{2}\big)}
        = \frac{\exp\big(-\frac{\|x-\tilde\mu_i\|^2}{2}\big)}{\sum_{j \in [k]} \exp\big(-\frac{\|x-\tilde\mu_j\|^2}{2}\big)}
        = [\tilde\phi(x)]_i, ~\forall i \in [k]
        \\
        \Rightarrow&~ \frac{\exp\big(-\frac{\|x-\mu_i\|^2}{2}\big)}{\exp\big(-\frac{\|x-\tilde\mu_i\|^2}{2}\big)} = \frac{\exp\big(-\frac{\|x-\mu_j\|^2}{2}\big)}{\exp\big(-\frac{\|x-\tilde\mu_j\|^2}{2}\big)}, \forall i,j \in [k]
        \\
        \Rightarrow&~ \|x - \mu_i\|^2 - \|x - \tilde\mu_i\|^2 = \|x - \mu_j\|^2 - \|x - \tilde\mu_j\|^2, ~\forall i, j \in [k]
        \\
        \Rightarrow&~ 2\big((\tilde\mu_i - \mu_i) - (\tilde\mu_j - \mu_j)\big)^\top x = \big(\|\mu_j\|^2 - \|\tilde\mu_j\|^2\big) - \big(\|\mu_i\|^2 - \|\tilde\mu_i\|^2\big), ~\forall i, j \in [k].
    \end{split}
    \end{equation}
    Since the left hand side is linear in $x \in \R^d$ and the right hand side is a constant, it must be that both sides are 0.
    That is, the necessary conditions for $\phi = \tilde\phi$ are that for any $i, j \in [k]$,
    1) $\tilde\mu_i - \mu_i = \tilde\mu_j - \mu_j$,
    and 2) $\|\mu_i\|^2 - \|\tilde\mu_i\|^2 = \|\mu_j\|^2 - \|\tilde\mu_j\|^2$.
    It can be checked that these two conditions are also sufficient for $\phi = \tilde\phi$.
    
    Denote $v := \mu_i - \tilde\mu_i$.
    The norms of the means are known and equal by Assumption \ref{assump:gmm_equal_norms}, which gives
    \begin{equation}\label{eq:gmm_2_1_matchPhi_ortho}
    \begin{split}
        \|\mu_i\|^2 - \|\tilde\mu_i\|^2 
        = \|\mu_i\|^2 - \|\mu_i - v\|^2
        = (2\mu_i-v)^\top v
        = 0, ~\forall i \in [k].
        \\
    \end{split}
    \end{equation}
    The last equality in equation \ref{eq:gmm_2_1_matchPhi_ortho} holds for a non-zero $v$ when the span of $\{2\mu_i-v: i \in [k]\}$ is $(d-1)$-dimensional subspace.
    On the other hand, the span of $\{2\mu_i - v: i \in [k]\}$ is at least $(k-1)$ by Assumption \ref{assump:gmm_full_rank}.
    When $d=k$, it must be that the dimension is exactly $(d-1)$, which means $v$ is an affine combination of $\{2\mu_i: i \in [k]\}$ by Claim \ref{clm:affine_comb}.
    
    
    Moreover, $v$ has to be orthogonal to $\{2\mu_i - v: i \in [k]\}$, which leads to the unique choice of $v$ that is the projection of the origin onto the $(d-1)$-dimensional subspace specified by the affine combinations of $\{2\mu_i: i\in[k]\}$.
    \ifx\coltSpace
    \vspace{-1em}
    \fi 
    \begin{claim}
        $v$ is the projection of the origin to the hyperplane defined by $\{2\mu_i: i\in[k]\}$, and is the only solution to equation \ref{eq:gmm_2_1_matchPhi_ortho}.
    \end{claim}
    \begin{proof}
    It is clear that this choice of $v$ satisfies $(2\mu_i - v)^\top v = 0$, $\forall i \in [k]$.
    To see that this is the unique choice, suppose there exists some $v'$ lying in the hyperplane of $\{2\mu_i\}$, and denote $\delta := v' - v$.
    
    Note that $\delta^\top v$ = 0:
    let the hyperplane specified by $\{2\mu_i\}_{i \in [k]}$ be specified as $\{x: \langle \vu, x\rangle = c\}$ for some $\vu \in \R^d$ and $c \in \R$.
    Then $v$, the projection of the origin, can be written as $v = \frac{c}{\|\vu\|} \cdot \frac{\vu}{\|\vu\|}$, i.e. $v$ is proportional to the normal vector $\vu$.
    %
    For any $v'$ in the hyperplane, it satisfy $\langle \vu, v'\rangle = c$, and 
    \begin{equation}
    \begin{split}
        \delta^\top v
        =& (v'-v)^\top v
        = \langle \frac{c}{\|\vu\|} \frac{\vu}{\|\vu\|}, v'\rangle 
            - \left\|\frac{c}{\|\vu\|} \frac{\vu}{\|\vu\|}\right\|^2
        \\
        =& \frac{c}{\|\vu\|^2} \cdot \langle \vu, v'\rangle - \frac{c^2}{\|\vu\|^2} \frac{\|\vu\|^2}{\|\vu\|^2}
        = \frac{c^2}{\|\vu\|^2} - \frac{c^2}{\|\vu\|^2}
        = 0.
    \end{split}
    \end{equation}
    
    Then for any $v'$ satisfying equation \ref{eq:gmm_2_1_matchPhi_ortho},
    \begin{equation}
    \begin{split}
        &(2\mu_i - v')^\top v'
        = (2\mu_i - v - \delta)^\top (v+\delta)
        \\
        =& \underbrace{(2\mu_i - v)^\top v}_{0} + 2\mu_i^\top \delta - \underbrace{v^\top \delta}_{0} - \underbrace{\delta^\top v}_{0} - \delta^\top \delta
        = (2\mu_i - \delta)^\top \delta
        = 0, ~\forall i \in [k].
    \end{split}
    \end{equation}
    Since $\{2\mu_i - \delta\}_{i \in [k]}$ spans the $(k-1)$-dimensional hyperplane and that $\delta$ lies in the hyperplane,
    it must be that $\delta = 0$, i.e. $v' = v$.
    \end{proof}
    
    Note that this choice of $v$ also satisfies $\|\mu_i - v\| = \|\mu_i\|$, since $v$ and the origin are reflections w.r.t. the hyperplane that is the affine hull of $\{\mu_i: i \in [k]\}$.
    In other words, $\{\mu_i - v\}_{i \in [k]}$ is related to $\{\mu_i\}_{i \in [k]}$ via the Householder transformation of the form $\mH := \mI_d - 2\frac{vv^\top}{\|v\|^2}$, i.e. $\mu_i - v = \mH\mu_i$.
    Denote $\hat{v} := \frac{v}{\|v\|_2}$.
    An explicit formula for $\hat{v}$ is $\hat{v} := \frac{\mMeans^{-\top} \1}{\sqrt{\1^\top\mMeans^{-1} \mMeans^{-\top}\1}}$.
    This finishes the proof for $d=k$.
    
    For $d > k$, the above argument still applies
    and $\mH$ remains the only indeterminacy (up to permutation),
    where $\mH := \mI_d - 2\hat{v}\hat{v}^\top$
    for $\hat{v} := \frac{(\mMeans^\dagger)^\top \1}{\sqrt{\1^\top\mMeans^\dagger(\mMeans^\dagger)^\top\1}}$.
    The reason is that even though the ambient dimension $d$ is larger, $\{\mu_i - v: i \in [k]\}$ has to have the same span as $\{\mu_i: i \in [k]\}$, since having the same predictor requires the column space of $\mMeans, \tilde\mMeans$ to match.
    Hence we only need to consider $v$ in the $k$-dimensional column space of $\mMeans$, which reduces to the case of $d=k$.
 
\end{proof}

\subsection{Proof for Theorem \ref{thm:gmm}: identifiability of predicting $x_{\tTwo}\otimes x_{\tThr}|x_\tOne$, \GHMM}
\label{sec:proof_gmm_triplets}

Similar to the discrete case, we will prove $x_2 \otimes x_3 | x_1$ and $x_1 \otimes x_3|x_2$ separately;
the proof for $x_1 \otimes x_2|x_3$ is analogous to $x_2 \otimes x_3 |x_1$ by symmetry and hence omitted.
The proofs also follow a similar strategy as in the proof for Theorem \ref{thm:discrete}, that is, to construct a 3-tensor using the predictor, on which applying Kruskal's theorem provides identifiability.

\paragraph{Case 1, $x_2 \otimes x_3 | x_1$:}
Let $\gX := \{x^{(i)} \in \R^d: i \in [k]\}$ be a linearly independent set,
and consider the following 3-tensor:
\begin{equation}
\begin{split}
    &\Tensor
    := \sum_{x_i \in \gX} x_1 \otimes \E[x_2 \otimes x_3|x_1]
    = \sum_{x_1 \in \gX} x_1 \otimes \E_{h_2|x_1}\big[ \E[x_2 \otimes x_3|x_1] | h_2\big]
    \\
    =& \sum_{x_1 \in \gX} x_1 \otimes \sum_{h_2} {P}(h_2|x_1) \E[x_2|h_2] \otimes \E[x_3|h_2]
    \\
    =& \sum_{i\in[k]} \sum_{x_1 \in \gX} {P}(h_2=i|x_1) x_1 \otimes \E[x_2|h_2=i] \otimes \E[x_3|h_2=i]
    \\
    =& \sum_{i\in[k]} \Big(\underbrace{\sum_{x_1} (\mTrans \phi(x_1))^\top \ei^{(k)} x_1}_{:= a_i} \Big) \otimes \mMeans_i \otimes (\mMeans\mTrans)_i.
    \\
\end{split}
\end{equation}
The matrices formed by second and third components are both of rank-$k$ by Assumption \ref{assump:gmm_full_rank}.
Hence in order to apply Kruskal's theorem on $\Tensor$, it suffices to show that there exists a choice of $\gX$ such that the matrix $\mA := [a_1, ..., a_k]$ is of rank $k$.
%
One such choice is to let $x^{(i)} = \mu_i$, which gives
\begin{equation}
\begin{split}
    a_i :=& \sum_{j \in [k]} \phi(x_1=\mu_j)^\top \mTrans^\top \ei^{(k)} \mu_j
    = \mMeans [\phi(\mu_1), ..., \phi(\mu_k)]^\top \mTrans^\top \ei^{(k)},
    \\
    \mA :=&~ [a_1, ..., a_k] = \mMeans [\phi(\mu_1), ..., \phi(\mu_k)]^\top \mTrans^\top.
\end{split}
\end{equation}
Since $\mMeans$, $\mTrans$ are both of rank $k$ by Assumption \ref{assump:gmm_full_rank}, we only need to argue that the matrix $\Phi := [\phi(\mu_1), ..., \phi(\mu_k)] \in \R^{k \times k}$ is of full rank.
Recall that for a mixture of $k$ Gaussians with identify covariance and mean $\{\mu_i \in \R^d: i \in [k]\}$,
the posterior function $\phi$ is defined entrywise as 
\begin{equation}
\begin{split}
    [\phi(x)]_i = \frac{\exp\big(-\frac{\|x -\mu_i\|_2^2}{2}\big)}{\sum_{j \in [k]} \exp\big(-\frac{\|x -\mu_j\|_2^2}{2}\big)}, ~\forall i \in [k].
\end{split}
\end{equation}
To show $\Phi$ is of full rank, we can equivalently show that a columnwise scaled version of $\Phi$ is full rank.
In particular, let's look at the matrix $\hat\Phi \in \R^{k \times k}$, where $\hat\Phi_{ij} = \exp(-\frac{\|\mu_i - \mu_j\|^2}{2})$;
that is, each column of $\hat\Phi$ can be considered as a scaled version of the column in $\Phi$ without the normalization for a unit $\ell_1$ norm.
It can be seen that $\hat{\Phi}$ is a Gaussian kernel matrix which is known to be full rank.

Therefore we have shown that each component of the tensor $\Tensor := \sum_{i \in [k]} a_i \otimes \mMeans_i \otimes (\mMeans\mTrans)_i$ has Kruskal rank $k$,
which allows to recover columns of $\mMeans, \mMeans\mTrans$ up to permutation and scaling by Kruskal's theorem.
The indeterminacy in scaling is further removed since the norms of $\{\mMeans_i\}_{i\in[d]}$ are known by Assumption \ref{assump:gmm_equal_norms}.

On the other hand, for any $\tilde\mMeans, \tilde\mTrans$ that form the same predictor as $\mMeans,\mTrans$,
$\Tensor$ can also be written as 
\begin{equation}
\begin{split}
    \Tensor 
    =& \sum_{x_1 \in \gX} x_1 \otimes \E[x_2 \otimes x_3|x_1]
    = \sum_{x_1 \in \gX} x_1 \otimes \tilde\E[x_2 \otimes x_3|x_1]
    \\
    =& \sum_{i\in[k]} \Big(\sum_{x_1} (\tilde\mTrans \tilde\phi(x_1))^\top \ei^{(k)} x_1 \Big) \otimes \tilde\mMeans_i \otimes (\tilde\mMeans\tilde\mTrans)_i.
    \\ 
\end{split}
\end{equation}
Hence columns of $\mMeans, \tilde\mMeans$ and $\mMeans\mTrans, \tilde\mMeans\tilde\mTrans$ are both matched up to a shared permutation, which proves identifiability.

\paragraph{Case 2, $\E[x_1 \otimes x_3|x_2]$:}
For the task of predicting $x_1,x_3$ given $x_2$, the predictor takes the form
\begin{equation}
\begin{split}
    \E[x_1\otimes x_3|x_2] = (\mEmiss\mTrans^\top) \diag(\phi(x_2)) (\mEmiss\mTrans)^\top.
\end{split}
\end{equation}
Let $\gX := \{\mu_i: i \in [k]\}$ as in the previous case,
and consider the 3-tensor
\begin{equation}
\begin{split}
    \Tensor
    :=& \sum_{x_2 \in \gX} x_2 \otimes \E[x_1\otimes x_3|x_2]
    = \sum_{x_2 \in \gX} x_2 \otimes \E_{h_2|x_2} (\E[x_1|h_2] \otimes \E[x_3|h_2])
    \\
    =& \sum_{h_2} \sum_{x_2 \in \gX} p(h_2|x_2) x_2 \otimes \E[x_1|h_2] \otimes \E[x_3|h_2]
    \\
    =& \sum_{i \in [k]} \Big(\underbrace{\sum_{x_2 \in \gX} (\phi(x_2))^\top \ei^{(k)} x_2}_{:= a_i}\Big) \otimes (\mMeans\mTrans^\top)_i \otimes (\mMeans\mTrans)_i,
\end{split}
\end{equation}
The first component is of rank-$k$ as shown in the proof for  $x_2\otimes x_3|x_1$, and the other two components are of rank-$k$ by Assumption \ref{assump:gmm_full_rank}.
Thus Kruskal's theorem applies and the columns of $\mMeans\mTrans, \mMeans\mTrans^\top$ are recovered up to a shared permutation.

The first component $\{a_i\}_{i \in [k]}$ are also recovered,
which means that if $\tilde\mMeans, \tilde\mTrans$ form the same predictor as $\mMeans,\mTrans$,
then for any linearly independent set $\gX$ with $k$ elements (not necessarily the previous choice of $\{\mu_i\}_{i\in[k]}$) such that $\gX$ leads to a full rank $\mA$,
we have $\mA = \tilde\mA$ where $\tilde\mA$ is parameterized by $\tilde\mMeans, \tilde\mTrans$.
For any such $\gX = \{x^{(i)}: i\in[k]\}$, denote $\mX := [x^{(1)}, ..., x^{(k)}]$,
then 
\begin{equation}
\begin{split}
    \mA = \mX [\phi(x^{(1)}), ..., \phi(x^{(k)})]^\top \mTrans^\top
    = \mX [\tilde\phi(x^{(1)}), ..., \tilde\phi(x^{(k)})]^\top \tilde\mTrans^\top
    = \tilde\mA.
    \\
\end{split}
\end{equation}
Since $\mX$ is of rank-$k$ by the choice of $\gX$, this means
\begin{equation}\label{eq:gmm_13_2_matchPhi}
\begin{split}
    [\tilde\phi(x^{(1)}), ..., \tilde\phi(x^{(k)})] = \underbrace{\tilde\mTrans^{-1}\mTrans}_{:= \mR} [\phi(x^{(1)}), ..., \phi(x^{(k)})]
    \Rightarrow 
    \tilde\phi(x^{(i)}) = \mR \phi(x^{(i)}), ~\forall i \in [k].
\end{split}
\end{equation}
%
Moreover, for any valid choice of $\gX$,
matrices defined with points in sufficiently small neighborhoods of $x^{(i)}$ are still of full rank by the upper continuity of matrix rank.
Hence the equality in equation \ref{eq:gmm_13_2_matchPhi} holds for points in these neighborhoods, and thus the Jacobian on both sides should be equal.
Then, the exact same proof of Lemma \ref{lem:gmm_matchPhi} applies, and we get $\tilde\phi$, $\phi$ are equal up to a permutation of coordinates.
Thus $\tilde\mMeans$ must be equal to (up to permutation) either $\mMeans$ or $\mH\mMeans$ for a Householder reflection $\mH$ by Lemma \ref{lem:gmm_phi_reflect}.
Finally, the solution of $\mH\mMeans$ is eliminated since it would lead to a $\tilde\mTrans$ that is not a valid stochastic matrix, as shown in the proof of Theorem \ref{thm:gmm_2_1_id}.

\subsection{Identifiability from pairwise conditional distribution}
\label{appendix:gmm_cond_distr}

We show that matching the entire conditional \textit{distribution} for \GHMM~provides identifiability. 
Though this is implied by Theorem \ref{thm:gmm_2_1_id}, which states that matching the conditional \textit{expectation} already suffices, having access to the full conditional distribution allows an even simpler proof. 
\ifx\coltSpace
\vspace{-0.8em}
\fi
\begin{theorem}[Identifiability of conditional distribution]
\label{thm:gmm_conditional}
    Let $\mMeans, \mTrans$ and $\tilde\mMeans, \tilde\mTrans$ be two set of parameters satisfying Assumption \ref{assump:latent} and \ref{assump:gmm_full_rank}.
    If $p(x_2|x_1; \mMeans,\mTrans) = p(x_2|x_1; \tilde\mMeans,\tilde{\mTrans})$, $\forall x_1,x_2 \in \R^d$,
    then $\mMeans = \tilde\mMeans$, $\mTrans = \tilde\mTrans$ up to a permutation of labeling. 
\end{theorem}
\ifx\coltSpace
\vspace{-0.8em}
\fi
\begin{proof}
    First note that the conditional distribution of $x_2$ given $x_1$ is a mixture of Gaussian, with means $\{\mu_i\}_{i \in [k]}$ and mixture weights given by $P(h_2|x_1) = \mTrans P(h_1|x_1)$,
    hence we can directly apply the identifiability of Gaussian mixtures to recover the means $\{\mu_i\}_{i \in [k]}$:
    \ifx\coltSpace\undefined
    \vspace{-0.8em}
    \fi
    \begin{lemma}[Proposition 4.3 in \cite{lindsay93}]
        \label{lem:gmm_id}
            Let $Q_k$ denote a Gaussian mixture with means $\{\xi_j\}_{j\in[k]} \in \R^d$.
            Suppose $\exists l \in [d]$ such that the set $\{[\xi_j]_l\}$ has distinct values,
            then one can recover $\{\xi_j\}_{j\in[k]}$ from moments of $Q_k$.
    \end{lemma}
    \vspace{-1em}
    We note that the assumption on the existence of a coordinate $l \in [k]$ is with out loss of generality, since we can first rotate the means to a different coordinate system in which this condition holds, then rotation back the means.
    Such rotation is guaranteed to exist since finding such rotation is equivalent to finding a vector $\vv$ s.t. $\vv^\top(\mu_i - \mu_j) \neq 0$ for every $i, j \in [k]$, for which the solution set is $\R^d \setminus \cup_{i,j \in [k]}\{\vu:\vu^\top(\mu_i-\mu_j)=0\} \neq \emptyset$.
    
    Recovering $\{\mu_i\}_{i \in [k]}$ means the scaled likelihood and the posterior both match, i.e. $\psi = \tilde\psi$, and $\phi(x) = P(h|x) = \frac{\psi}{\|\psi\|_1}$.
    The conditional distribution is
    \begin{equation}
    \begin{split}
        p(x_2|x_1)
        = \sum_{i,j\in[k]} p(x_2|h_2) p(h_2|h_1) p(h_1|x_1)
        = \frac{1}{(2\pi)^{d/2}} \psi(x_2)^\top \mTrans \phi(x_1).
    \end{split}
    \end{equation}
    Choose a set $\gX := \{x^{(i)}\}_{i \in [k]}$ such that $\Psi_\gX := [\psi(x^{(1)}), ..., \psi(x^{(k)})] \in \R^{k \times k}$ is full rank.
    $\Phi_\gX := [\phi(x^{(1)}), ..., \phi(x^{(k)})] \in \R^{k \times k}$ is also full rank since its columns are nonzero scalings of columns of $\Psi_\gX$.
    Then we have
    \begin{equation}
    \begin{split}
        \Psi_\gX^\top \mTrans \Phi_\gX
        = \tilde\Psi_{\gX}^\top \tilde\mTrans \tilde\Phi_\gX
        = \Psi_\gX^\top \tilde\mTrans \Phi_\gX 
        \Rightarrow \mTrans = \tilde\mTrans.
    \end{split}
    \end{equation} 
\end{proof}





\section{Missing proofs for HMM}
\label{appendix:hmm}

\subsection{Example for Theorem \ref{thm:discrete_nonId}}
\label{appendix:hmm_eg_12_21_13_31}

  We provide a concrete example for the non-identifiability of predicting $x_2|x_1$, $x_1|x_2$, $x_3|x_1$, and $x_1|x_3$.
  Let $d = 4$, $k=3$:
  {\small
  \begin{equation}
  \begin{split}
    &\mEmiss = \begin{bmatrix}
      0.23016003 & 0.3549092  & 0.16493077 \\
      0.30716059 & 0.06962305 & 0.37321636 \\
      0.2580854  & 0.26965425 & 0.22226035 \\
      0.20459398 & 0.3058135  & 0.23959252
    \end{bmatrix},
    ~\tilde\mEmiss = \begin{bmatrix}
      0.24120928 & 0.35062535 & 0.15816537 \\
      0.28937626 & 0.07433156 & 0.38629218 \\
      0.26077674 & 0.26749114 & 0.22173212 \\
      0.20863772 & 0.30755194 & 0.23381033
    \end{bmatrix},
    \\
    &\mTrans = \begin{bmatrix}
      0.56893146 & 0.35811118 & 0.07295736 \\
      0.35811118 & 0.10805638 & 0.53383243 \\
      0.07295736 & 0.53383243 & 0.39321021
    \end{bmatrix},
    ~\tilde{\mTrans} = \begin{bmatrix}
      0.59740926 & 0.30452087 & 0.09806987 \\
      0.30452087 & 0.1331689  & 0.56231024 \\
      0.09806987 & 0.56231024 & 0.33961989
    \end{bmatrix},
    \\
    & \det(\mEmiss) = \det(\tilde\mEmiss) = 0.0110,
    \det(\mTrans) = \det(\tilde\mTrans) = -0.1611.
  \end{split}
  \end{equation}
  }
  Note that $\mTrans, \tilde\mTrans$ are both symmetric, which means this is also a valid counter example for learning to predict all of $x_2|x_1$, $x_1|x_2$, $x_3|x_1$, $x_1|x_3$.

\subsection{Proof of Claim \ref{clm:rotation_sumRowCol}}
\label{appendix:hmm_claim_rotSum}

Recall that $\mR \in \R^{k \times k}$ for $k=3$ represents a rotation whose rotation axis is $\frac{1}{3}e_1 + \frac{1}{3}e_2 + \frac{1}{3}e_3$,
and we would like to show that rows and columns of $\mR$ both sum up to 1:
\ifx\coltSpace\undefined
\vspace{-0.8em}
\fi
\begin{claim*}[Claim \ref{clm:rotation_sumRowCol} restated]
    Each row and each column of $\mR$ sum up to 1.
\end{claim*}
\begin{proof}
  Denote the $d$-dimensional simplex by $\simplex_d$, i.e. $\simplex_d := \{x\in \R^d: \sum_{i\in[d]} x_i = 1\}$,
  and let $\gP_c := \{v \in \R^d: \sum_{i\in[d]} v_i =c\}$ for some positive constant $c$ denote a hyperplane parallel to the hyperplane in which probability vectors lie.
  
  Let's first check that the columns of $\mR$ sum up to 1.
  Any $v \in \gP_c$ can be written as $v = c\cdot [\alpha_1, \alpha_2, ..., \alpha_{d-1}, 1-\sum_{i\in[d-1]}\alpha_i]$ for some $[\alpha_1, ..., \alpha_{d-1}] \in \simplex_{d-1}$.
  Let $r_i$ denote the $i_{th}$ row of $\mR$,
  then $\mR v \in \gP_c$ means $\sum_{i \in [d]} \langle r_i, v \rangle = \langle \sum_{i \in [d]} r_i, v \rangle = c$.
  Let $\beta_j$ denote the $j_{th}$ coordinate of $\sum_{i\in[d]} r_i$, then
  \begin{equation}
  \begin{split}
    &\sum_{i\in[d-1]} \beta_i\alpha_i + \beta_d \big(1 - \sum_{i \in [d-1]}\alpha_i\big) = 1,~ \forall [\alpha_1, ..., \alpha_{d-1}] \in \simplex_{d-1}
    \\
    \Rightarrow& \sum_{i \in [d-1]} (\beta_i - \beta_d) \alpha_i + \beta_d = 1, ~ \forall [\alpha_1, ..., \alpha_{d-1}] \in \simplex_{d-1}
    \\
    \Rightarrow&~ \beta_i = 1,~ \forall i \in [d].
  \end{split}
  \end{equation}
  It then follows that $\mR^{-1} = \mR^\top$ also has columns summing up to 1, since 
  \begin{equation}
  \begin{split}
    \sum_{i\in[d]} (\mR\mR^{-1})_{ij}
    = \langle\sum_{i\in[d]} r_i, (\mR^{-1})_j\rangle
    = \langle \1, (\mR^{-1})_j\rangle
    = 1,~\forall j \in [d].
  \end{split}
  \end{equation}

\end{proof}

\section{Nonidentifiability from large time gaps}

As noted earlier, there is an inherent obstacle when using prediction tasks on tokens that are more than 1 time gaps apart.
For instance, if we are predicting $x_{t+1}$ given $x_1$ for some $t > 1$ with \GHMM, then we are still able to identify $\mMeans$ from the posterior function, however it remains to to recover $\mTrans$ from $\mTrans^{t}$. 
For general matrices, it is clear that matching a power of a matrix does not imply the matrix itself is matched.
For our case, even though requiring $\mTrans$ to be stochastic adds additional constraints, matching the matrix power still does not suffice to identify the underlying matrix, 
as formalized in the following claim.
\ifx\coltSpace\undefined
\vspace{-0.8em}
\fi
\begin{claim}[Nonidentifiability of matrix powers]
\label{clm:transPower_nonId}
    For any positive integer $t$, there exist stochastic matrices $\mTrans, \tilde\mTrans$ satisfying Assumption \ref{assump:latent}, \ref{assump:discrete_full_rank},
    such that $\mTrans \neq \tilde{\mTrans}$ and $\mTrans^t = \tilde\mTrans^{t}$. 
\end{claim}
\begin{proof}
    We will set $\tilde\mTrans$ to be a rotation of $\mTrans$, that is, $\tilde\mTrans = \mR \mTrans$ for some matrix $\mR$ that implicitly performs a rotation.
    Precisely, using notations for the \GHMM~setup, set $a \in [0,1]$, and let
    the parameters $(\mTrans,\mMeans)$ be given by
    \begin{equation*}
    \begin{split}
        \mTrans =& \begin{bmatrix}a & 0 & 1-a \\ 1-a & a & 0 \\ 0 & 1-a & a \end{bmatrix},
        ~\mMeans = \begin{bmatrix}1 & -1/2 & -1/2 \\ 0 & -\sqrt{3}/2 & \sqrt{3}/2 \\ 1/\sqrt{2} & 1/\sqrt{2} & 1/\sqrt{2} \end{bmatrix}.
        \\
    \end{split}
    \end{equation*}
    Let $\theta$ be some rotation angle, and denote by $\mR(\theta) := \begin{bmatrix} \cos(\theta) & -\sin(\theta) & 0 \\ \sin(\theta) & \cos(\theta) & 0 \\ 0 & 0 & 1 \end{bmatrix}$ a rotation that acts on the first two dimensions.
    We will show that for any $\theta \in \R$, we have 
    \begin{equation}\label{eq:trans_power_commute}
    \begin{split}
    \tilde{\mTrans}
    := \big(\mMeans^{-1} \big(\mR(\theta)\big)^{-1} \mMeans\big) \cdot \mTrans
        = \mTrans \cdot \big(\mMeans^{-1} \big(\mR(\theta)\big)^{-1} \mMeans\big).
    \end{split}
    \end{equation}
    Assuming \eqref{eq:trans_power_commute}, 
    since $\mR(\theta)$ represents a rotation of angle $\theta$, $\big(\mR(\theta)\big)^{\tau}$ corresponds to a rotation of angle $\tau\theta$ for any integer $\tau$ ($\tau$ could be negative).
    Setting $\theta := \frac{2\pi}{t}$, we then have 
    \begin{equation}
    \begin{split}
    \tilde\mTrans^t
    =& \big(\mMeans^{-1} \big(\mR(\theta)\big)^{-1} \mMeans \cdot \mTrans\big)^t
    = \mTrans^{t} \big(\mMeans^{-1} \big(\mR(\theta)\big)^{-1} \mMeans)\big)^{t}
    = \mTrans^t \mMeans^{-1} \big(\mR(\theta)\big)^{-t} \mMeans
    \\
    =& \mTrans^t \mMeans^{-1} \cdot \mR(2\pi) \cdot \mMeans
    = \mTrans^t.
    \end{split}
    \end{equation}
    
    For $\tilde\mTrans$ to serve as a valid example for our theorem, it remains to check that for every $t$, there exists a choice of $a$ such that $\tilde\mTrans := \mR\mTrans$, where $\mR := \mMeans^{-1}\big(\mR(\frac{2\pi}{t})\big)^{-1}\mMeans$, is a valid stochastic matrix.
    That is, $\tilde\mTrans$ has 1) columns and rows each summing up to 1, and 2) entries bounded in $[0,1]$.
    Let's first show that the columns and rows each sum up to 1.
    Noting that $\mMeans^{-1} = \frac{1}{3}\begin{bmatrix}
        2 & 0 & \sqrt{2} \\ -1 & -\sqrt{3} & \sqrt{2} \\ -1 & \sqrt{3} & \sqrt{2}
    \end{bmatrix}$,
    the column sums are
    \begin{equation}
    \begin{split}
        \1^\top\tilde{\mTrans}
        = \1^\top \mMeans^{-1}\mR(\theta)^{-1}\mMeans \mTrans
        \eqLabel{1}{=} \1^\top \mTrans \mMeans^{-1}\mR(\theta)^{-1}\mMeans
        = \sqrt{2} \ve_3^\top \mR(\theta) \mMeans
        = \sqrt{2} \ve_3^\top \mMeans
        = \sqrt{2} \frac{1}{\sqrt{2}} \1 
        = \1,
    \end{split}
    \end{equation}
    where step $\rnum{1}$ uses \eqref{eq:trans_power_commute}.
    Similarly, the row sums are 
    \begin{equation}
    \begin{split}
        \tilde\mTrans\1
        = \mMeans^{-1}\mR(\theta)^{-1}\mMeans \1 
        = \mMeans^{-1}\mR(\theta)^{-1} \cdot \frac{3}{\sqrt{2}} \ve_3
        = \mMeans^{-1} \cdot \frac{3}{\sqrt{2}} \ve_3
        = \1.
    \end{split}
    \end{equation}
    
    To show that there exists a choice of $\mTrans$ such that entries of $\tilde{\mTrans}$ are non-negative, we provide a concrete example where $\mTrans$ is defined with $a = \frac{1}{2}$.
    It can be checked that $\tilde\mTrans := \mMeans^{-1}(\mR(\frac{2\pi}{t}))^{-1}\mMeans$ has non-negative entries for $t \in \{2, 3, 4, ..., 10\}$.
    For larger $t$, let $\theta = \frac{2\pi}{t}$,
    then we have by the Taylor expansion of $\mR(\frac{2\pi}{t})$:
    \begin{equation}
    \begin{split}
        \mR(\theta)
        :=& \begin{bmatrix}
        \cos\theta & -\sin\theta & 0 \\ \sin\theta & \cos\theta & 0 \\ 0 & 0 & 1 \\ \end{bmatrix}
        = \begin{bmatrix} 1-\theta^2/2 + c_1 \theta^4 & -\theta + c_2 \theta^2 & 0 \\ \theta + c_2 \theta^2 & 1-\theta^2/2 + c_1\theta^4 & 0 \\ 0 & 0 & 1 \\ \end{bmatrix}
        \\
        =& \mI
            + \theta \begin{bmatrix} 0 & -1 & 0 \\ 1 & 0 & 0 \\ 0 & 0 & 0 \\\end{bmatrix}
            + \theta^2 \begin{bmatrix} -1/2 + c_1 \theta^2 & c_2 & 0 \\ c_2 & -1/2 + c_1\theta^2 & 0 \\ 0 & 0 & 0 \\ \end{bmatrix}
    \end{split}
    \end{equation}
    for some constants $c_1 \in [-\frac{1}{4!}, \frac{1}{4!}]$, $c_2 \in [-\frac{1}{2}, \frac{1}{2}]$.
    Substituting this into $\tilde\mTrans := \mMeans^{-1}\mR(\theta)^{-1}\mMeans$ gives
    {\small
    \begin{equation}
    \begin{split}
        \tilde\mTrans
        =& \begin{bmatrix}
            a - \frac{\theta}{\sqrt{3}}(1-a) & \frac{\theta}{\sqrt{3}}(1-2a) & 1 - a + \frac{\theta}{\sqrt{3}} a 
            \\
            1 - a + \frac{\theta}{\sqrt{3}} a & a - \frac{\theta}{\sqrt{3}}(1-a) & \frac{\theta}{\sqrt{3}}(1-2a)
            \\
            \frac{\theta}{\sqrt{3}}(1-2a) & 1 - a + \frac{\theta}{\sqrt{3}} a & a - \frac{\theta}{\sqrt{3}}(1-a)
            \end{bmatrix}
            \\
            &\ + \frac{1}{3} \begin{bmatrix}
                -1 + 2c_1 \theta^2 & \frac{1}{2}-c_1\theta^2 - \sqrt{3}c_2 & \frac{1}{2} - c_1\theta^2 + \sqrt{3}c_2
                \\
                \frac{1}{2} - c_1\theta^2 - \sqrt{3}c_2 & -1 + 2c_1\theta^2 + \sqrt{3}c_2 & \frac{1}{2}-c_1\theta^2 
                \\
                \frac{1}{2} - c_1\theta^2 + \sqrt{3} c_2 & \frac{1}{2}-c_1\theta^2 & -1+2c_1\theta^2 - \sqrt{3}c_2
            \end{bmatrix}
            \cdot \begin{bmatrix}
                a & 0 & 1-a \\ 1-a & a & 0 \\ 0 & 1-a & a \\
            \end{bmatrix}
        \\
        =& \frac{1}{2}\begin{bmatrix}
            1 - \frac{\theta}{\sqrt{3}} & 0 & 1 + \frac{\theta}{\sqrt{3}} 
            \\
            1 + \frac{\theta}{\sqrt{3}} & 1 - \frac{\theta}{\sqrt{3}} & 0
            \\
            0 & 1 + \frac{\theta}{\sqrt{3}} & 1 - \frac{\theta}{\sqrt{3}}
            \end{bmatrix}
            + \frac{\theta^2}{6} \begin{bmatrix}
                -\frac{1}{2} + c_1\theta^2 - \sqrt{3}c_2 & 1-2c_1\theta^2 & -\frac{1}{2}+c_1\theta^2 + \sqrt{3}c_2
                \\
                -\frac{1}{2} + c_1\theta^2 & -\frac{1}{2} + c_1\theta^2 + \sqrt{3}c_2 & 1 - 2c_1\theta^2 -\sqrt{3}c_2
                \\
                1 - 2c_1\theta^2 + \sqrt{3}c_2 & -\frac{1}{2} + c_1\theta^2 - \sqrt{3}c_2 & -\frac{1}{2} + c_1\theta^2
            \end{bmatrix}
        \\
        \eqLabel{1}{\geq}& \frac{1}{2}\begin{bmatrix}
            1 - \frac{\theta}{\sqrt{3}} & 0 & 1 + \frac{\theta}{\sqrt{3}} 
            \\
            1 + \frac{\theta}{\sqrt{3}} & 1 - \frac{\theta}{\sqrt{3}} & 0
            \\
            0 & 1 + \frac{\theta}{\sqrt{3}} & 1 - \frac{\theta}{\sqrt{3}}
            \end{bmatrix}
            + \theta^2 \begin{bmatrix}
                -0.25 & -0.16 & -0.25 \\
                -0.09 & -0.25 & 0.01 \\ 
                0.01 & -0.25 & -0.09 \\
            \end{bmatrix}
    \end{split}
    \end{equation}}
    
    where the inequality $\rnum{1}$ is taken entry-wise.
    It can be checked that all entries are non-negative for $\theta \leq \frac{2\pi}{10}$.
    

    
    \paragraph{Proof of \eqref{eq:trans_power_commute}}
    Let's conclude the proof by proving \eqref{eq:trans_power_commute}.
    Denote $\mR_2(\theta) := \begin{bmatrix}
        \cos(\theta) & -\sin(\theta) \\ \sin(\theta) & \cos(\theta)
    \end{bmatrix}$,
    i.e. $\mR(\theta) = \begin{bmatrix}
        \mR_2(\theta) & 0 \\ 0 & 1 \\
    \end{bmatrix}$.
    Denote $\mU := \begin{bmatrix}
        1 & -1/2 & -1/2 \\ 0 & -\sqrt{3}/2 & \sqrt{3}/2
    \end{bmatrix}$,
    i.e. $\mMeans = \begin{bmatrix}
        \mU \\ \1^\top / \sqrt{2}
    \end{bmatrix}$.
    We can write 
    \begin{equation}
    \begin{split}
        \mMeans^\top\mR(\theta)^\top \mMeans
        = \begin{bmatrix}\mU^\top & \1/\sqrt{2}\end{bmatrix}
            \begin{bmatrix}\mR_2(\theta)^\top & 0 \\ 0 & 1 \end{bmatrix}
            \begin{bmatrix}\mU \\ \1^\top/\sqrt{2}\end{bmatrix}
        = \mU^\top \mR_2(\theta)^\top \mU + \frac{\1\1^\top}{2}.
    \end{split}
    \end{equation}
    Let $\mR_2(\theta)$ denote a clockwise rotation of angle $\theta$, then 
    \begin{equation}\label{eq:trans_power_U}
    \begin{split}
        \mU
        = [v_1, \mR_2\big(\frac{2\pi}{3}\big)v_1, \mR_2\big(\frac{4\pi}{3}\big)v_1]
        = [\mR_2\big(\frac{4\pi}{3}\big)v_2, v_2, \mR_2\big(\frac{2\pi}{3}\big)v_2]
        = [\mR_2\big(\frac{2\pi}{3}\big)v_3, \mR_2\big(\frac{4\pi}{3}\big)v_3, v_3],
    \end{split}
    \end{equation}
    where $v_1 = \begin{bmatrix}1 \\ 0\end{bmatrix}$,
    $v_2 = \begin{bmatrix}-1/2 \\ -\sqrt{3}/2\end{bmatrix}$,
    $v_3 = \begin{bmatrix}-1/2 \\ \sqrt{3}/2\end{bmatrix}$.
    Denote $\alpha_{ij} := v_i^\top\mR_2^\top v_j$ for $i, j \in [3]$.
    Noting $\mTrans = a\mI + (1-a)\begin{bmatrix}
        0 & 0 & 1 \\ 1 & 0 & 0 \\ 0 & 1 & 0 \\
        \end{bmatrix} := a\mI + (1-a)\mP$,
    we have 
    \begin{equation}
    \begin{split}
        &\mMeans^\top\mR(\theta)^\top\mMeans\mTrans = \mTrans\mMeans^\top\mR(\theta)^\top\mMeans
        \\
        \Leftrightarrow&~
        \mU^\top \mR_2(\theta)^\top \mU (a\mI + (1-a)\mP) + \frac{\1\1^\top}{2}\mTrans
        = (a\mI + (1-a)\mP)\mU^\top \mR_2(\theta)^\top \mU  + \mTrans\frac{\1\1^\top}{2}
        \\
        \eqLabel{1}{\Leftrightarrow}&~
        \mU^\top\mR_2(\theta)^\top \mU \mP = \mP \mU^\top\mR_2(\theta)^\top \mU
        \\
        \Leftrightarrow&~
        \begin{bmatrix}
            \alpha_{31} & \alpha_{32} & \alpha_{33} \\
            \alpha_{11} & \alpha_{12} & \alpha_{13} \\
            \alpha_{21} & \alpha_{22} & \alpha_{23} \\
        \end{bmatrix}
        \overset{(*)}{=} \begin{bmatrix}
            \alpha_{12} & \alpha_{13} & \alpha_{11} \\
            \alpha_{22} & \alpha_{23} & \alpha_{21} \\
            \alpha_{32} & \alpha_{33} & \alpha_{31} \\
        \end{bmatrix}.
    \end{split}
    \end{equation}
    where step $\rnum{1}$ uses $\1\1^\top\mTrans = \mTrans\1\1^\top = \1\1^\top$.
    The equality $(*)$ is true due to \eqref{eq:trans_power_U}.
\end{proof}



\end{document}